\documentclass{article}
\usepackage{lipsum}
\usepackage{amsfonts,amssymb}
\usepackage{bbm}
\usepackage{graphicx}
\usepackage{epstopdf}
\usepackage{algorithmic}
\usepackage{enumitem}
\setlist[enumerate]{leftmargin=.5in}
\setlist[itemize]{leftmargin=.5in}
\usepackage{amsopn}
 \usepackage{enumitem}   
\usepackage{multirow}
\usepackage{subcaption}
 
\usepackage{amsmath,amsthm}
\newtheorem{theorem}{Theorem}
 \newtheorem{definition}{Definition}
\newtheorem{proposition}{Proposition}

\newcommand{\RN}[1]{%
  \textup{\uppercase\expandafter{\romannumeral#1}}%
}

\DeclareMathOperator*{\argmin}{arg\,min}

\newcommand{\eqdef}{\stackrel{\vartriangle}{=}}
\newtheorem{remark}{Remark}
\newtheorem{corollary}{Corollary}

\title{Multi-Kernel Regression with Sparsity Constraint }

\author{Shayan Aziznejad\thanks{ Biomedical Imaging Group, EPFL, Lausanne, Switzerland (shayan.aziznejad@epfl.ch,   michael.unser@epfl.ch)  This work was funded by the Swiss National Science Foundation under Grant 200020\textunderscore 184646 / 1. }
\and Michael Unser\footnotemark[1]  }

\begin{document}

\maketitle

\begin{abstract}
In this paper, we provide a Banach-space formulation of supervised learning with  generalized total-variation (gTV) regularization. We identify the class of kernel functions that are admissible in this framework. Then, we propose a  variation of supervised learning in a continuous-domain hybrid search space with gTV regularization. We show that the solution  admits a multi-kernel expansion with adaptive positions. In this representation, the number of active kernels is upper-bounded by the number of data points while the gTV regularization imposes an $\ell_1$ penalty on the kernel coefficients. Finally, we illustrate numerically  the outcome of our theory.
\end{abstract}


\section{Introduction}

 The determination of  an unknown function  from a series of   samples is a classical problem in machine learning. It falls under the category of ``supervised learning,'' for which there exists a  rich literature  (see \cite{bishop2006pattern,hastie2009overview,wahba1990spline} for classical textbooks, as well as \cite{cucker2007learning,devroye2013probabilistic,gyorfi2006distribution,steinwart2008support} for more recent ones).   The goal of supervised learning is to recover a target function $f:\mathbb{R}^d\rightarrow \mathbb{R}$ from its  $M$  noisy samples  ${y}_m = f(\boldsymbol{x}_m)+\epsilon_m$, $m=1,2,\ldots,M$.   The disturbance terms $\epsilon_m$ are typically assumed to be i.i.d.\ samples of a   zero-mean probability law ({\it e.g.}, additive Gaussian noise) while the input vectors $\boldsymbol{x}_m$ are either assumed to be in the random or fixed design \cite[Section 1.9]{gyorfi2006distribution}.  
 
 A general way to formulate supervised learning is through the   minimization problem
\begin{equation}\label{Pb:GeneralFormLearning}
\min_{f }  \bigg( \underbrace{ \sum_{m=1}^M \mathrm{E}(f(\boldsymbol{x}_m),y_m)  }_{\RN{1}} +  \lambda \underbrace{\vphantom{\sum_{m=1}^M \mathrm{E}(f(\boldsymbol{x}_m),y_m)  }\mathcal{R}(f)}_{\RN{2}}\bigg),
\end{equation}
where the  cost function is made of two terms. The first one (data fidelity)  measures how well $f$ fits the given training dataset while the second one (regularization)    imposes the prior knowledge about the function model. The parameter $\lambda \in \mathbb{R}^{+}$  balances the terms. 

\subsection{RKHS in Machine Learning}

The simplest form of \eqref{Pb:GeneralFormLearning} is the least-squares problem with Tikhonov regularization  
\begin{equation}\label{Pb:L2Tikhonov}
\min_{f\in \mathcal{H}_\mathrm{L}(\mathbb{R}^d) } \left(\sum_{m=1}^M |f(\boldsymbol{x}_m)- {y}_m|^2 + \lambda \| \mathrm{L}\{f\}\|_{L_2}^2\right),
\end{equation}
where $\mathrm{L}$ is the regularization operator  and   $\mathcal{H}_{\mathrm{L}}(\mathbb{R}^d)$, known as the native space of $\mathrm{L}$,  is the  space of    functions $f:\mathbb{R}^d\rightarrow\mathbb{R}$ such that $\mathrm{L}\{f\}\in L_2(\mathbb{R}^d)$   (see \eqref{Eq:LpNorm} for the definition of the $L_p$ spaces). It is a classical quadratic minimization problem that has a closed-form solution \cite{tikhonov1963solution}. An important assumption in this formulation is the continuity of the sampling functionals $\delta_{\boldsymbol{x}_m}=\delta(\cdot-\boldsymbol{x}_m):f \mapsto f( \boldsymbol{x}_m)$ for $m=1,2,\ldots,M$. This is equivalent to  $\mathcal{H}_\mathrm{L}(\mathbb{R}^d)$ being a reproducing-kernel Hilbert space (RKHS)  \cite{aronszajn1950theory,de1966splines,wahba1990spline}, which is a key concept in supervised learning \cite{berlinet2011reproducing,steinwart2008support}.   

The Hilbert space $\mathcal{H}(\mathbb{R}^d)$  consisting of functions from $\mathbb{R}^d$ to $\mathbb{R}$ is called an RKHS  if  there exists a bivariate symmetric and positive-definite  function $\mathrm{k}:\mathbb{R}^d\times\mathbb{R}^d \rightarrow\mathbb{R}$  such that, for all $\boldsymbol{x}\in \mathbb{R}^d$,  $\mathrm{k}(\boldsymbol{x},\cdot) \in \mathcal{H}(\mathbb{R}^d)$ and  $f(\boldsymbol{x})= \langle \mathrm{k}(\boldsymbol{x},\cdot) , f(\cdot) \rangle_\mathcal{H}$ \cite{aronszajn1950theory}. The function $\mathrm{k}(\cdot,\cdot)$ is   unique and is called the reproducing kernel of $\mathcal{H}(\mathbb{R}^d)$. 

The supervised learning over the RKHS $\mathcal{H}(\mathbb{R}^d)$ can be formulated through the minimization 
 \begin{equation}\label{Pb:L2RKHS}
\min_{f\in \mathcal{H}(\mathbb{R}^d)} \left(\sum_{m=1}^M \mathrm{E}(f(\boldsymbol{x}_m), {y}_m) + \lambda \| f\|_\mathcal{H}^2\right).
\end{equation} 
The kernel representer theorem states that the solution of \eqref{Pb:L2RKHS} admits the form
\begin{equation}\label{Eq:RKHSSolution}
f(\cdot) = \sum_{m=1}^M a_m \mathrm{k}(\cdot,\boldsymbol{x}_m)
\end{equation}
for some appropriate weights $a_m\in\mathbb{R}$, where $m=1,2,\ldots,M$   \cite{kimeldorf1971some,scholkopf2001generalized}. The  expansion \eqref{Eq:RKHSSolution} is the key element of  kernel-based schemes in machine learning  \cite{scholkopf2001learning,shawe2004kernel,vapnik1998statistical}   and, in particular, support-vector machines (SVM) \cite{evgeniou2000regularization,steinwart2008support}. Moreover,  optimal rates have been derived for learning using the expansion \eqref{Eq:RKHSSolution} in several setups \cite{caponnetto2007optimal,mendelson2010regularization,steinwart2009optimal}, particularly for  Gaussian kernels \cite{eberts2013optimal}.  Computing the RKHS norm of a function $f$ of the form  \eqref{Eq:RKHSSolution} results  in $\|f\|_\mathcal{H}^2 = \boldsymbol{a}^T \mathbf{G} \boldsymbol{a}$, where $\mathbf{G}\in \mathbb{R}^{M\times M}$  is a symmetric and positive-definite matrix  with $[\mathbf{G}]_{m,n} = \mathrm{k}(\boldsymbol{x}_m,\boldsymbol{x}_n)$. It is called the Gram matrix of the kernel $\mathrm{k}(\cdot,\cdot)$. The practical outcome of this observation is that  the infinite-dimensional problem \eqref{Pb:L2RKHS} over the space of functions $\mathcal{H}(\mathbb{R}^d)$ becomes equivalent to the finite-dimensional problem \cite{scholkopf2001generalized}
\begin{equation}\label{Pb:DiscreteL2}
\min_{\boldsymbol{a}\in\mathbb{R}^M} \left( \sum_{m=1}^M \mathrm{E}([\mathbf{G} \boldsymbol{a}]_m, {y}_m) +  \lambda \boldsymbol{a}^T \mathbf{G} \boldsymbol{a}\right),
\end{equation}
which is    of size $M$ and can be computed     numerically. 

\subsection{Toward  Sparse Kernel Expansions}
 In the solution form \eqref{Eq:RKHSSolution}, the kernels are  shifted to the location of the data samples. This is elegant but  can become cumbersome when the number of samples $M$ grows large.   Several schemes have been  developed  to reduce the number of active kernels. One proposed approach is to   use  a sparsity-enforcing loss such as the $\epsilon$-insensitive norm of SVM regression \cite{steinwart2003sparseness,steinwart2004sparseness,steinwart2009sparsity}. Another approach is to replace the  quadratic regularization $\boldsymbol{a}^T \mathbf{G} \boldsymbol{a}$ in the reduced finite-dimensional problem \eqref{Pb:DiscreteL2} by a sparsity-promoting penalty such as $\|\boldsymbol{a}\|_{\ell_1}= \sum_{m=1}^M |a_m|$. This results in   \eqref{Pb:DiscreteL2} becoming
\begin{equation}\label{Pb:GeneralizedLASSO}
\min_{\boldsymbol{a}\in\mathbb{R}^M} \left(\sum_{m=1}^M \mathrm{E}([\mathbf{G} \boldsymbol{a}]_m, {y}_m) +  \lambda \|\boldsymbol{a}\|_{\ell_1}\right),
\end{equation} 
  which is called   the generalized LASSO \cite{roth2004generalized}.   The properties of this estimator have been studied both from a statistical  \cite{shi2011concentration} and approximation-theoretical point of view \cite{wang2013approximation}. 

 In this paper, we consider a  Banach-space formulation of  supervised learning.   We  choose  the generalized total-variation (gTV) norm as the regularization term in order to promote  sparsity  in the continuous domain. The effect of gTV regularization has been extensively studied in the context of linear inverse problems \cite{fisher1975spline,mammen1997locally,unser2017splines}. For an invertible operator $\mathrm{L}$  (see Definition \ref{Def:KernelAdmis}), the gTV norm is defined as 
 \begin{equation}\label{Eq:gTVdef}
 \mathrm{gTV}(f) = \|\mathrm{L}\{f\}\|_{\mathcal{M}},
 \end{equation}
 where  $\mathcal{M}(\mathbb{R}^d)$ is the space of bounded Radon measures (see \eqref{Eq:Mdef} for a precise definition) and   $\|\cdot\|_{\mathcal{M}}$ is the total-variation norm in the sense of measures \cite{rudin2006real}.    
 
 One can formulate supervised learning with gTV regularization through the minimization 
\begin{equation}\label{Pb:SingleKernelgTV}
\min_{f\in\mathcal{M}_{\mathrm{L}}(\mathbb{R}^d)} \left(\sum_{m=1}^M \mathrm{E}(f(\boldsymbol{x}_m), {y}_m) + \lambda \|\mathrm{L}\{f \} \|_{\mathcal{M}}\right),
\end{equation}
where $\mathcal{M}_{\mathrm{L}}(\mathbb{R}^d)$ is the native Banach space of the operator   $\mathrm{L}:\mathcal{M}_{\mathrm{L}}(\mathbb{R}^d)\rightarrow\mathcal{M}(\mathbb{R}^d)$ equipped with the gTV norm (see Definition \ref{Def:NativeSpace}). The fact that  $\mathcal{M}_{\mathrm{L}}(\mathbb{R}^d)$ is   a Banach space ({\it i.e.}, a complete normed space) follows from the invertibility of ${\rm L}$  (see Theorem \ref{Thm:NativeSpace}). 
A consequence of the general representer theorem of  \cite{unser2017splines} is that there is always a solution of \eqref{Pb:SingleKernelgTV} that admits  a linear kernel expansion  of the form 
\begin{equation}\label{Eq:SingleTVSolForm}
f(\cdot)=\sum_{l=1}^{M_0} a_l \mathrm{k}(\cdot,\boldsymbol{z}_l), 
\end{equation}
for some  unknown integer $M_0\leq M$, non-zero  kernel weights $a_l\in \mathbb{R}$, and some  distinct  adaptive kernel positions $\boldsymbol{z}_l\in\mathbb{R}^d$ \cite{gupta2018continuous}. There, the function $\mathrm{k}(\cdot,\cdot):\mathbb{R}^d\times\mathbb{R}^d \rightarrow \mathbb{R}$   is the shift-invariant kernel associated to  the Green's function of the operator $\mathrm{L}$. In other words, we have that $\mathrm{k}(\boldsymbol{x},\boldsymbol{y}) = \rho_{\mathrm{L}}(\boldsymbol{x}-\boldsymbol{y})$, where $\rho_{\mathrm{L}}=\mathrm{L}^{-1}\{\delta\}$. 

 There   exist  works on supervised learning over Banach spaces, especially via the concept of reproducing-kernel Banach spaces (RKBS)  \cite{fasshauer2015solving,zhang2009reproducing,zhang2012regularized}. However, there are several differences between RKBS and our proposed scheme of learning with gTV regularization. Firstly, as highlighted in \cite{zhang2012regularized}, the RKBS  representer theorem yields a nonlinear kernel expansion  for the optimal solution. Secondly, its kernel positions necessarily coincide with the data points. Last but not least, the Banach spaces in the RKBS theory are restricted to reflexive one (see Section \ref{Subsec:FuncSpc} for the definition of reflexive Banach spaces), which excludes the case of learning with gTV regularization that is known to enforce sparsity in the continuous domain.  

Let us also mention that a  formulation with strong link to \eqref{Pb:SingleKernelgTV} has been presented in \cite{bach2017breaking} for learning a function $f:\mathbb{R}^d\rightarrow\mathbb{R}$ from  a continuously indexed family of atoms $\{\mathrm{k}_{\boldsymbol{z}}\}_{\boldsymbol{z}\in\mathcal{V}}$, where $\mathcal{V}$ is a compact topological space.  Putting it in a similar form as \eqref{Pb:SingleKernelgTV}, the proposed formulation in \cite{bach2017breaking} for supervised learning is   equivalent to the  minimization
\begin{equation}\label{Pb:BachSynth}
\min_{ \mu  \in \mathcal{M}(\mathcal{V})}\left( \sum_{m=1}^M \mathrm{E}\left( \int_{\mathcal{V}} \mathrm{k}_{\boldsymbol{z}}(\boldsymbol{x}_m) \mathrm{d}\mu(\boldsymbol{z}), y_m\right) + \lambda \|\mu\|_{\mathcal{M}}\right), 
\end{equation}
where $\mathcal{M}(\mathcal{V})$ is the space of Radon measures  over $\mathcal{V}$. The relevant property there is that   the minimization of \eqref{Pb:BachSynth} introduces an atomic measure $\mu=\sum_{l=1}^{M_0} a_l \delta(\cdot-\boldsymbol{z}_l)$. It hence suggests the parametric form  \eqref{Eq:SingleTVSolForm}  with $\mathrm{k}(\cdot,\boldsymbol{z}_{l})=\mathrm{k}_{\boldsymbol{z}_l}(\cdot)$ for the learned function.

The minimization problem \eqref{Pb:BachSynth} is  a synthetis-based  formulation for supervised learning where the basis functions are known {\it a priori}, in contrary to \eqref{Pb:SingleKernelgTV} which is an analysis-based formalism that relies on regularization theory in Banach spaces. Interestingly, the two formulations are equivalent when the family of atoms in \eqref{Pb:BachSynth} coincides with the class of shifted Green's function of the  regularization operator $\mathrm{L}$; that is, $\mathrm{k}_{\boldsymbol{z}}(\cdot) = \rho_{\mathrm{L}}(\cdot-\boldsymbol{z})$.

To conclude this section, we discuss the connection between \eqref{Pb:SingleKernelgTV} and generalized LASSO. One readily   verifies  that the gTV norm enforces an $\ell_1$ penalty on the kernel coefficients $a_l$. More precisely, the expansion \eqref{Eq:SingleTVSolForm} translates the original  problem \eqref{Pb:SingleKernelgTV} into the discrete minimization 
 \begin{equation}\label{Pb:gTVdiscreteLASSO}
\min_{\boldsymbol{a}\in\mathbb{R}^M , \mathrm{Z}\in \mathbb{R}^{d\times M_0}} \left(\sum_{m=1}^M \mathrm{E}( [\mathbf{G}_{\mathrm{Z}} \boldsymbol{a}]_m ,y_m) + \lambda \|\boldsymbol{a}\|_{\ell_1}\right),
\end{equation}
where  $\mathrm{Z}= \begin{pmatrix}\boldsymbol{z}_1, \boldsymbol{z}_2 , \ldots, \boldsymbol{z}_{M_0}
\end{pmatrix}$ is the kernel-position  matrix and  $\mathbf{G}_{\mathrm{Z}} \in \mathbb{R}^{M \times M_0} $ is a matrix with $[\mathbf{G}_{\mathrm{Z}}]_{m,l}= \mathrm{k}(\boldsymbol{x}_m,\boldsymbol{z}_l)$. The reduced problem \eqref{Pb:gTVdiscreteLASSO} can be seen as an extended version  of the  generalized LASSO in \eqref{Pb:GeneralizedLASSO}. The fundamental difference is that the minimization is through the positions as well.

 \subsection{Multi-Kernel Schemes}
The solution forms  \eqref{Eq:RKHSSolution} and \eqref{Eq:SingleTVSolForm}  heavily depend on the kernel function $\mathrm{k}(\cdot,\cdot)$. Hence, choosing the proper kernel is a challenging task  that  requires careful consideration.  One can use a cross-validation scheme in order to compare the performance of several kernel estimators and select the best one for the desired application \cite{gonen2011multiple}. Another approach is to  learn a new  kernel function $\mathrm{k}_{\boldsymbol{\mu}} = \sum_{n=1}^N \mu_n \mathrm{k}_n$ from a family of given kernels $\mathrm{k}_1 ,\mathrm{k}_2,\ldots, \mathrm{k}_{N}$  \cite{bach2004multiple,lanckriet2004learning,rakotomamonjy2008simplemkl,micchelli2005learning}.  This transforms the original problem \eqref{Pb:DiscreteL2} into  the joint optimization  
\begin{equation}\label{Pb:MKL}
\min_{{\boldsymbol{\mu}} \in\mathbb{R}^N, \boldsymbol{a}\in\mathbb{R}^M}\left(  \sum_{m=1}^M \mathrm{E}([\mathbf{G}_{\boldsymbol{\mu}}  \boldsymbol{a}]_m, {y}_m) +  \lambda \boldsymbol{a}^T \mathbf{G}_{\boldsymbol{\mu}}  \boldsymbol{a} + \mathrm{R}(\boldsymbol{\mu})\right),
\end{equation}
where $\mathbf{G}_{\boldsymbol{\mu}}$ is the Gram matrix of the learned kernel $\mathrm{k}_{\boldsymbol{\mu}}$ and $\mathrm{R}(\cdot)$ regularizes the coefficient vector $\boldsymbol{\mu}$, for example like in $\mathrm{R}(\boldsymbol{\mu}) = \|\boldsymbol{\mu}\|_{\ell_p}=\left(\sum_{n=1}^N |\mu_n|^p\right)^{\frac{1}{p}}$ for $ 1\leq p \leq 2$   \cite{bach2008consistency,bazerque2013nonparametric,gao2010sparse,kloft2009efficient,kloft2010unifying}. The learned function will then take  the  generic form   
\begin{equation}
f(\cdot)= \sum_{n=1}^N \sum_{m=1}^M \mu_n a_m \mathrm{k}(\cdot,\boldsymbol{x}_m).
\end{equation} 

\subsection{Our Contribution}
In this paper, we provide a Banach-space framework for supervised learning with gTV regularization.   We study the topological structures of the search space of this problem and we characterize the  class of admissible regularization operators together with their associated kernel functions. 

We also propose a multi-kernel extension of supervised learning with gTV regularization. To that end, we consider the minimization
\begin{equation}\label{Pb:MKRIntro}
\min_{\stackrel{ f_n\in\mathcal{M}_{\mathrm{L}_n}(\mathbb{R}^d),}{ f=\sum_{n=1}^N f_n} } \left( \sum_{m=1}^M \mathrm{E}({f}(\boldsymbol{x}_m),{y}_m)+ \lambda \sum_{n=1}^N \|\mathrm{L}_n\{{f}_n\}\|_\mathcal{M}\right). 
\end{equation}
In this formulation, the target function $f$ is   decomposed into $N$ additive components, where  the smoothness of each component  has been expressed by its corresponding regularization operator. 
  Our main result, which follows from Theorem \ref{Thm:Main}, is the existence of  a solution  of \eqref{Pb:MKRIntro} that yields a  multi-kernel expansion of the target function and that takes the form
   \begin{equation}\label{Eq:MKRSolIntro}
f(\cdot) = \sum_{n=1}^{N}\sum_{l=1}^{M_n}  a_{n,l} \mathrm{k}_n(\cdot,\boldsymbol{z}_{n,l}), \quad \|\boldsymbol{a}\|_{\ell_0} \leq M,
 \end{equation} 
where  $\|\boldsymbol{a}\|_{\ell_0}$ is called the $\ell_0$ norm  of $\boldsymbol{a}$ and is equal to the number of nonzero elements of $\boldsymbol{a}$,  and $\mathrm{k}_n$ is the   shift-invariant kernel associated to the operator $\mathrm{L}_n$. Moreover, the total number of nonzero coefficients is upper-bounded by the number  $M$ of data points and, hence, is not growing with the number $N$ of components.    We also illustrate numerically  the effect of using multiple kernels.

\subsection{Roadmap} 
The paper is organized as follows:   We  present some mathematical preliminaries in Section 2. In Section 3, we study the Banach  space  structure of the native spaces and we characterize the class of admissible kernels.  We propose and prove our main result  in Section 4. Finally, we   provide further discussions and illustrations in Section 5.

\section{Preliminaries}

In this section, we recall   relevant mathematical concepts  such as the  function spaces that we use   throughout the paper  along with properties of  linear operators that are defined over those spaces. 

\subsection{Function Spaces}\label{Subsec:FuncSpc}
All the derivatives of  a rapidly decaying function decay faster than the inverse of any polynomial at infinity.  Then, a smooth and slowly growing function    is an element of $\mathcal{C}^{\infty}(\mathbb{R}^d)$ such that all of its derivatives have asymptotic growth controlled by a polynomial. Finally, a heavy-tailed function $f:\mathbb{R}^d\rightarrow\mathbb{R}$ satisfies $f(\boldsymbol{x}) \geq C(1+ \|\boldsymbol{x}\|)^\alpha$ for some finite constants $C,\alpha >0$.

 For $p\in [1,\infty)$, we denote by $L_p(\mathbb{R}^d)$, the Banach space of measurable  functions $f:\mathbb{R}^d\rightarrow \mathbb{R}$ with finite $L_p$ norm, {\it i.e.} 
\begin{equation}\label{Eq:LpNorm}
L_p(\mathbb{R}^d)= \left\{f:\mathbb{R}^d\rightarrow\mathbb{R} \text{ measurable}: \|f\|_{L_p} \eqdef \left(\int_{\mathbb{R}^d} |f(\boldsymbol{x})|^p {\rm d}\boldsymbol{x} \right)^{\frac{1}{p}}<+\infty\right\}.
\end{equation} 
The Schwartz  space of smooth and rapidly decaying functions   $\varphi:\mathbb{R}^d \rightarrow \mathbb{R}$ is denoted by $\mathcal{S}(\mathbb{R}^d)$. Its topological dual is $\mathcal{S}'(\mathbb{R}^d)$,  the space of tempered distributions  \cite{im1964generalized}.  We remark  that any smooth and slowly growing function  $f:\mathbb{R}^d\rightarrow\mathbb{R}$ specifies the continuous linear functional $\varphi\mapsto \int_{\mathbb{R}^d} f(\boldsymbol{x})\varphi(\boldsymbol{x})\mathrm{d}\boldsymbol{x}$ over $\mathcal{S}(\mathbb{R}^d)$ and, hence, is an element of $\mathcal{S}'(\mathbb{R}^d)$. 

The space of   continuous functions  over $\mathbb{R}^d$ that vanish at infinity is     $\mathcal{C}_0(\mathbb{R}^d)$. It is a Banach space  equipped with the supremum norm $\|\cdot\|_\infty$. The space of Schwartz functions   $\mathcal{S}(\mathbb{R}^d)$ is  densely embedded  in $\mathcal{C}_0(\mathbb{R}^d)$. Hence,   the topological dual of $\mathcal{C}_0(\mathbb{R}^d)$ can be defined as 
\begin{align}\label{Eq:Mdef}
\mathcal{M}(\mathbb{R}^d) = \{ w\in \mathcal{S}'(\mathbb{R}^d): \quad \|w\|_{\mathcal{M}} \eqdef \sup_{\substack{\varphi \in \mathcal{S}(\mathbb{R}^d)\\ \|\varphi\|_\infty=1}} |\langle w,\varphi \rangle|<+\infty\}.
\end{align}
In fact, $\mathcal{M}(\mathbb{R}^d)$ is the Banach space of bounded Radon measures over $\mathbb{R}^d$   equipped with the total-variation norm $\|\cdot\|_\mathcal{M}$ \cite{rudin2006real}. It includes  the shifted Dirac impulses $\delta(\cdot-\boldsymbol{x}_0)$, with $\|\delta(\cdot-\boldsymbol{x}_0)\|_\mathcal{M}=1$. Moreover,  ${L}_1(\mathbb{R}^d) \subseteq \mathcal{M}(\mathbb{R}^d)$    with the relation   $\|f\|_{L_1}=\|f\|_{\mathcal{M}}$ for all $f\in L_1(\mathbb{R}^d)$.  This allows one to  interpret $(\mathcal{M}(\mathbb{R}^d),\|\cdot \|_\mathcal{M})$ as a generalization  of  $({L}_1(\mathbb{R}^d),\|\cdot\|_{L_1})$.  

For a Banach space  $\mathcal{X}$, we consider two   topologies for its continuous dual space $\mathcal{X}'$.  The first one is the strong  topology. It is induced from the dual norm in the sense that a sequence $\{w_n \}_{n=0}^{\infty} \in \mathcal{X}'$ is said to converge in the strong topology to $w^*\in \mathcal{X}'$  if $\lim_{n\rightarrow\infty} \| w_n - w^*\|_{\mathcal{X}'} = 0$. The second one is   the weak*-topology that comes from the predual space $\mathcal{X}$ in the sense that    a sequence $\{w_n \}_{n=0}^{\infty}$ is said to converge   in the weak*-topology to $w^*$  if, for any element $\varphi \in \mathcal{X}$, $\{\langle w_n , \varphi \rangle \}_{n=0}^{\infty}$ converges to $\langle w^* , \varphi \rangle$.

Finally, let us mention that any Banach space $\mathcal{X}$ is isometrically isomorphic to a closed subspace of its  second dual $\mathcal{X}'' = \left(\mathcal{X}'\right)'$ (see, for example, \cite[pp. 95]{rudin1991functional}). For the sake of simplicity, we make the possible embedding mappings implicit in our framework. This   leads to writing the latter proposition simply, via the inclusion $\mathcal{X}\subseteq \mathcal{X}''$. In this regard, a Banach space is reflexive if  we have that $\mathcal{X}= \mathcal{X}''$. Typical examples of reflexive Banach spaces are $L_p(\mathbb{R}^d)$ spaces for $p\in(1,\infty)$. By contrast, the space $\mathcal{C}_0(\mathbb{R}^d)$ and, consequently, its dual $\mathcal{M}(\mathbb{R}^d)$, are not reflexive.
\subsection{Linear Operators}
\label{Sec:LinOp}
The linear operator $\mathrm{L}:\mathcal{S}(\mathbb{R}^d) \rightarrow \mathcal{S}'(\mathbb{R}^d)$ is called shift-invariant  if,  for any   function ${\varphi}\in \mathcal{S}(\mathbb{R}^d)$ and  any shift value   $\boldsymbol{x}_0\in \mathbb{R}^d$, we have that 
\begin{equation}\label{Eq:LSIdefSch}
\mathrm{L}\{ \varphi(\cdot- \boldsymbol{x}_0 )\}=\mathrm{L}\{\varphi\}(\cdot- \boldsymbol{x}_0).
\end{equation}
 We recall a variant of the celebrated Schwartz kernel theorem for linear and shift-invariant (LSI) operators (see \cite{simon1971distributions} for a ``simple'' proof of the general case).

\begin{theorem}[Schwartz kernel theorem]
 For any LSI operator $\mathrm{L}:\mathcal{S}(\mathbb{R}^d)\rightarrow\mathcal{S}'(\mathbb{R}^d)$, there exists a unique  distribution  $h \in \mathcal{S}'(\mathbb{R}^d)$, known as the impulse response of $\mathrm{L}$, such that 
\begin{equation}\label{SchwartzKernelTheorem}
\forall \varphi \in \mathcal{S}(\mathbb{R}^d):  \mathrm{L}\{\varphi \}(\cdot)= \int_{\mathbb{R}^d}   h(\cdot-\boldsymbol{y})\varphi (\boldsymbol{y}) \mathrm{d}\boldsymbol{y}.
\end{equation}
\end{theorem} 

In this paper, we restrict ourselves to the class of continuous LSI operators that have an extended domain and are defined over the space of tempered distributions $\mathcal{S}'(\mathbb{R}^d)$. One can fully characterize this class in   the Fourier domain. The Fourier transform is a well-defined and continuous operator over $\mathcal{S}'(\mathbb{R}^d)$ and is denoted by $\mathcal{F}: \mathcal{S}'(\mathbb{R}^d) \rightarrow \mathcal{S}'(\mathbb{R}^d)$. Consequently, the frequency response of the LSI operator $\mathrm{L}:\mathcal{S}'(\mathbb{R}^d)\rightarrow\mathcal{S}'(\mathbb{R}^d)$ is defined as the Fourier transform of its impulse response 
\begin{equation}\label{Eq:FourierLSI}
\widehat{\mathrm{L}}(\boldsymbol{\omega}) \eqdef \mathcal{F}\{ \mathrm{L}\{ \delta\} \} (\boldsymbol{\omega}).
\end{equation}
It is known that the frequency response of any continuous LSI operator over $\mathcal{S}'(\mathbb{R}^d)$ is a smooth and slowly growing function \cite{schwartz1957theorie}. Additionally, any smooth and slowly growing function $\widehat{\mathrm{L}}(\boldsymbol{\cdot})$ defines an LSI and continuous operator $\mathrm{L}: \mathcal{S}'(\mathbb{R}^d)\rightarrow\mathcal{S}'(\mathbb{R}^d)$ via 
\begin{equation}\label{Eq:LSIFourier}
\mathrm{L}\{f\} = \mathcal{F}^{-1}\lbrace \widehat{\mathrm{L}}  \widehat{f} \rbrace.
\end{equation}
Typical examples of such operators are polynomials of derivative in dimension $d=1$ and polynomials of the Laplacian operator for $d>1$ \cite{duchon1977splines}.  
\section{Banach-Space Kernels}
In this section, we introduce our Banach-space framework of learning with   gTV regularization.    We start by defining the class of kernel-admissible operators. 
  
\begin{definition}\label{Def:KernelAdmis}
The  linear operator $\mathrm{L}:\mathcal{S}'(\mathbb{R}^d)\rightarrow \mathcal{S}'(\mathbb{R}^d)$ is called kernel-admissible (or simply admissible) if 
\begin{enumerate}[label=(\roman*)]
\item it is shift-invariant\footnote{Although the notion of shift-invariant operators in \eqref{Eq:LSIdefSch} is defined for operators acting on Schwartz functions, one can extend it by duality to those whose domain is $\mathcal{S}'(\mathbb{R}^d)$. For more details on extension by duality, we refer to \cite[Section 3.3.2]{unser2014introduction}.};  \label{C:SI}
\item it is an isomorphism over $\mathcal{S}'(\mathbb{R}^d)$, meaning that it is continuous and invertible, its inverse being the continuous operator $\mathrm{L}^{-1}:\mathcal{S}'(\mathbb{R}^d)\rightarrow\mathcal{S}'(\mathbb{R}^d)$; \label{C:ISO} 
\item the sampling functional $\delta_{\boldsymbol{x}_0}:{f}\mapsto {f}(\boldsymbol{x}_0)$ is weak*-continuous in the topology of its native space (see Definition \ref{Def:NativeSpace} and Theorem \ref{Thm:NativeSpace}).   \label{C:Weak}
\end{enumerate}
\end{definition} 

The restriction to LSI operators is not crucial to our framework. However, it lends itself to the convenience of  an  analysis in the Fourier domain. It also allows us to    provide necessary and sufficient conditions to characterize the class of admissible operators (see Theorem \ref{Thm:KernelAdmisOp}).  The invertibility  assumption, on the other hand, is essential  to have decaying kernels; that is, to have   ${\rm k}(\boldsymbol{x}-\boldsymbol{y})\rightarrow 0$ whenever $\|\boldsymbol{x}-\boldsymbol{y}\|\rightarrow \infty$. In fact, it is known that the Green's function of any LSI  operator with a nontrivial null space necessarily has a singularity in the Fourier domain at the origin  \cite{unser2014introduction}. Finally, the assumption of  the (weak*) continuity of the sampling functional is a natural choice in learning theory. The main motivation here is to guarantee the (weak*) lower semicontinuity of the global cost functional in \eqref{Pb:MKRIntro}. This can be used, together with the generalized Weierstrass theorem, to prove  the existence of solutions (see Theorem \ref{Thm:Main}). Let us note that  the definition of weak*-continuity depends on the Banach structure of the native space. In the sequel, we first  properly define native spaces and then specify their underlying Banach structures. 
 
\begin{definition}\label{Def:NativeSpace}
The native space of the LSI isomorphism $\mathrm{L}:\mathcal{S}'(\mathbb{R}^d)\rightarrow \mathcal{S}'(\mathbb{R}^d)$ is  the  pre-image  of $\mathrm{L}$ over the space of bounded Radon measures; that is, the space $\mathcal{M}_{\mathrm{L}}(\mathbb{R}^d)= \mathrm{L}^{-1}\{ \mathcal{M}(\mathbb{R}^d)\}$. 
\end{definition}

Theorem \ref{Thm:NativeSpace} summarizes the important properties of the native spaces.  Its proof is available in  Appendix A.
\begin{theorem}\label{Thm:NativeSpace}
Let $\mathrm{L}:\mathcal{S}'(\mathbb{R}^d ) \rightarrow \mathcal{S}'(\mathbb{R}^d )$ be an LSI  isomorphism over $\mathcal{S}'(\mathbb{R}^d  )$. Then, its native space  is a topological vector space with the following properties:
\begin{enumerate}[label=(\roman*)]
\item It is a Banach space equipped with the generalized total-variation norm \label{It:MLBanach}
\begin{equation}
\mathrm{gTV}(f) = \|f\|_{\mathcal{M}_\mathrm{L}} \eqdef \|\mathrm{L}\{f \}\|_{\mathcal{M}}.
\end{equation}

\item The restriction of $\mathrm{L}$ to its native space results in the  isomorphism $\mathrm{L}:\mathcal{M}_{\mathrm{L}}(\mathbb{R}^d )  \rightarrow \mathcal{M}(\mathbb{R}^d )$. \label{It:MLM}

\item The adjoint operator $\mathrm{L}^*$ is well-defined over $\mathcal{C}_0(\mathbb{R}^d )$  and its image is the Banach space $\mathcal{C}_{\mathrm{L}}(\mathbb{R}^d )$ with the norm  $\|f\|_{\mathcal{C}_\mathrm{L}}\eqdef\|\mathrm{L}^{-1*}\{f\}\|_\infty$.  \label{It:Adjoint}
\item The space $\mathcal{C}_{\mathrm{L}}(\mathbb{R}^d )$  is the predual of $\mathcal{M}_\mathrm{L}(\mathbb{R}^d )$, meaning that $(\mathcal{C}_\mathrm{L}(\mathbb{R}^d ))'=\mathcal{M}_\mathrm{L}(\mathbb{R}^d )$. \label{It:Predual}

\item The space of Schwartz  functions is   embedded in the native space. Moreover, the native space itself is  densely embedded in the space of tempered distributions. The embedding hierarchy is indicated as   \label{It:Embedding}
\begin{equation}
\mathcal{S}(\mathbb{R}^d )  {\hookrightarrow} \mathcal{M}_{\mathrm{L}}(\mathbb{R}^d ) \stackrel{d.}{\hookrightarrow}  \mathcal{S}'(\mathbb{R}^d ).
\end{equation}
\end{enumerate}
\end{theorem}

  \begin{figure}[t]
\begin{minipage}{1.0\linewidth}
  \centering
  \centerline{\includegraphics[width=\linewidth]{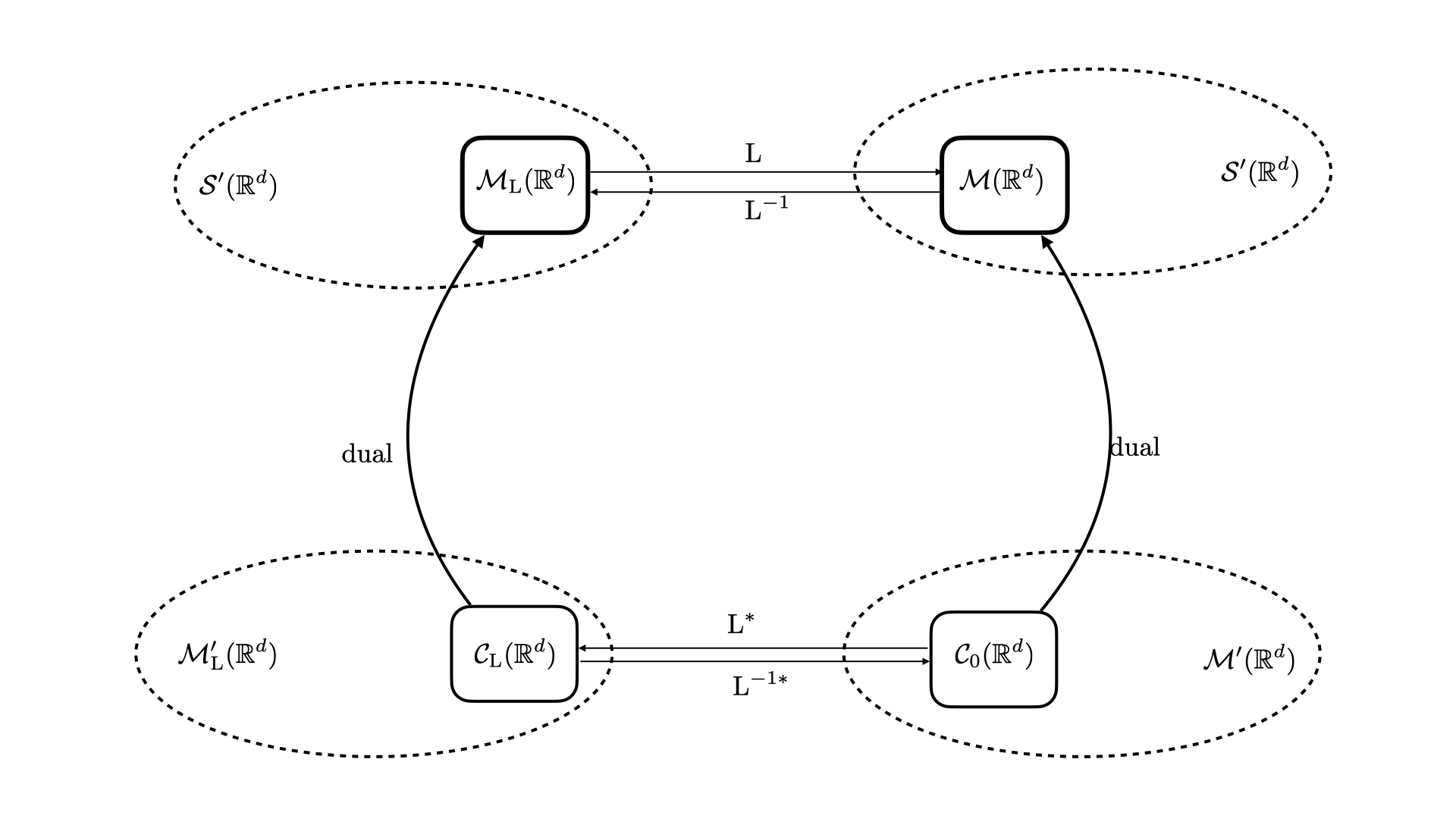}}
  \caption{  A Schematic diagram that illustrates the Banach spaces of interest.}\label{Fig:Diagram} \medskip
\end{minipage}
\end{figure}
  
We have summarized the Banach spaces and the mappings between them in Figure \ref{Fig:Diagram}. Due to Theorem \ref{Thm:NativeSpace}, the weak*-continuity of the sampling functional (Condition \ref{C:Weak} in Definition \ref{Def:KernelAdmis})  is equivalent to the   inclusion of the shifted Dirac impulses  in the predual of the native space. In other words, for all $\boldsymbol{x}_0 \in \mathbb{R}^d$, one should have  that  $\delta(\cdot-\boldsymbol{x}_0)\in\mathcal{C}_{\mathrm{L}}(\mathbb{R}^d)$. 

We now define the shift-invariant  kernel  associated to an admissible operator.  
\begin{definition}\label{Def:Kernel}
The   shift-invariant kernel  associated to the admissible operator $\mathrm{L}:\mathcal{S}'(\mathbb{R}^d)\rightarrow \mathcal{S}'(\mathbb{R}^d)$ is the bivariate function $\mathrm{k}:\mathbb{R}^d\times\mathbb{R}^d\rightarrow\mathbb{R}$ with  $\mathrm{k}(\boldsymbol{x},\boldsymbol{y})=\rho_{\mathrm{L}}(\boldsymbol{x}-\boldsymbol{y})$, where $\rho_{\mathrm{L}}=\mathrm{L}^{-1}\{\delta\}$ is the Green's function of $\mathrm{L}$.  
\end{definition}
In Theorem \ref{Thm:KernelAdmisOp}, we provide the  necessary and sufficient conditions that characterize  the class of admissible LSI operators. The proof can be found in Appendix B.
\begin{theorem}\label{Thm:KernelAdmisOp}
Let $\mathrm{L}$ be an admissible operator. Then, its associated Green's function $\rho_{\mathrm{L}}=\mathrm{L}^{-1}\{\delta\}:\mathbb{R}^d\rightarrow\mathbb{R}$  satisfies the following properties:
\begin{enumerate}[label=(\roman*)]
\item  It is   a continuous function that vanishes at infinity.  In other words, $\rho_{\mathrm{L}}\in\mathcal{C}_0(\mathbb{R}^d)$.  \label{It:kernelcont}
\item Its Fourier transform  $\widehat{\rho_\mathrm{L}}(\boldsymbol{\omega})$ is a smooth, non-vanishing, slowly growing, and heavy-tailed function of $\boldsymbol{\omega}$. \label{It:Slowgrowth}
\end{enumerate}
Additionally, any function $\rho:\mathbb{R}^d \rightarrow\mathbb{R}$ that satisfies these properties can be appointed to an  admissible operator $\mathrm{L}:\mathcal{S}'(\mathbb{R}^d)\rightarrow\mathcal{S}'(\mathbb{R}^d)$  defined as 
\begin{equation}
\mathrm{L}\{ f  \} = \mathcal{F}^{-1} \left\{\frac{\widehat{f}(\boldsymbol{\omega})}{ \widehat{\rho}(\boldsymbol{\omega})} \right\}.
\end{equation}
\end{theorem}
\begin{remark}
We have stated in Section \ref{Sec:LinOp}  a well-known result by Schwartz that determines the general family  of LSI  operators (not necessarily invertible) over $\mathcal{S}'(\mathbb{R}^d)$. In Theorem \ref{Thm:KernelAdmisOp}, particularly via Condition \ref{It:Slowgrowth}, we are excluding the noninvertible members of this family. Hence,   Condition \ref{It:Slowgrowth} fully characterizes the class of linear isomorphisms over $\mathcal{S}'(\mathbb{R}^d)$. 
\end{remark}
Using Theorem \ref{Thm:KernelAdmisOp}, we now draw a connection to the  well-known class of reproducing kernels, which are constrained to be symmetric (because of their positive-definiteness). 
\begin{corollary}\label{Corollary}
Any symmetric admissible kernel (in the sense of Theorem \ref{Thm:KernelAdmisOp}) is a shift-invariant reproducing kernel up to multiplication by a sign factor. 
\end{corollary}
\begin{proof}
Let ${\rm k}(\cdot,\cdot)$ be a symmetric and shift-invariant admissible kernel.  Then, the corresponding Green's function ${\rho}_{\rm L}$ is also a symmetric function and, hence, its Fourier transform $\widehat{\rho}_{\rm L}(\boldsymbol{\omega})$ is a real function that is also smooth and non-vanishing. Hence, the sign of $\widehat{\rho}_{\rm L}(\boldsymbol{\omega})$  is constant everywhere. By multiplying with a sign factor, we can  then assume that $\widehat{\rho}_{\rm L}(\boldsymbol{\omega})$ is positive everywhere. Now by invoking Bochner's theorem (see, for example, \cite[Appendix B]{unser2014introduction}), we  deduce that ${\rho}_{\rm L}$ is a positive-definite function which, together with the symmetric assumption, implies that  ${\rm k}(\cdot,\cdot)$     is indeed a reproducing kernel.
\end{proof}

The practical implication of Theorem \ref{Thm:KernelAdmisOp} is that it yields  Fourier-domain criteria to determine the admissibility of an operator $\mathrm{L}$.  In particular,  and due to the Rieman-Lebesgue lemma, if $\widehat{\rho_\mathrm{L}}$ is an absolutely integrable function then condition \ref{It:kernelcont} holds. 

As the last part of this section, we  use this characterization to introduce  some families of admissible kernels. Our first example is made of super-exponential kernels  defined as 
\begin{equation}
\mathrm{k}_{\alpha} ( \boldsymbol{x},\boldsymbol{y}) = \exp(-\| \boldsymbol{x}-\boldsymbol{y} \|_{\alpha}^\alpha), \quad \alpha\in(0,2),
\end{equation}
where $\|\boldsymbol{x}\|_{\alpha}= \left( \sum_{i=1}^d |x_i|^\alpha \right)^{\frac{1}{\alpha}}$ for any $\boldsymbol{x}=(x_i)\in\mathbb{R}^d$. These functions are known to be positive-definite \cite[Appendix B]{unser2014introduction}. Their inverse Fourier transforms (the so-called $\alpha$-stable distributions) are heavy-tailed and infinitely smooth, with algebraically decaying derivatives of any order \cite[Chapter 5]{sato1999levy}. Hence, they satisfy the   conditions of Theorem \ref{Thm:KernelAdmisOp}. Note that the classical Gaussian kernels are excluded because their frequency responses are not heavy-tailed. However, one can get arbitrarily close  by letting  $\alpha$ tend to its critical value $2$.  Moreover, there are arguments  in regularized RKHS  that  support the use of Gaussian kernels. For example, in   \cite{girosi1993priors,smola1998connection,yuille1998motion}, the Gaussian RKHS has been implicitly characterized  by using the Taylor expansion of the corresponding regularization operator. Further, \cite{steinwart2006explicit} uses the notion of holomorphic functions to explicitly characterize Gaussian RKHS.   We conjecture that the present Banach-space formulation can be extended to cover Gaussian kernels as well. However, this requires one to consider a space larger  than $\mathcal{S}'(\mathbb{R})$. 

 Our second example is made of  Bessel potentials  used in kernel estimation  \cite{aronszajn1961theory}. For a positive real number $s>d$, we consider the operator $(\mathrm{I}-\Delta)^{\frac{s}{2}}:\mathcal{S}'(\mathbb{R}^d)\rightarrow\mathcal{S}'(\mathbb{R}^d)$, where $\Delta$ is the Laplacian operator.  The Bessel potentials  are  the Green's function of such operators. They correspond to the shift-invariant kernels
\begin{equation}\label{Eq:BesselPot}
G_s(\boldsymbol{x},\boldsymbol{y})= \mathcal{F}^{-1}\left\lbrace\frac{1}{(1+\|\boldsymbol{\omega}\|_2^2)^{\frac{s}{2}}}\right\rbrace(\boldsymbol{x}-\boldsymbol{y}).
\end{equation} 
Clearly, the function $\frac{1}{(1+\|\boldsymbol{\omega}\|_2^2)^{\frac{s}{2}}}$ is in $L_1(\mathbb{R}^d)$ for $s>d$. By invoking the Riemann-Lebesgue lemma, we deduce that its inverse Fourier transform is a continuous function that vanishes at infinity. Hence, the kernel function $G_s(\cdot,\cdot)$ satisfies Property \ref{It:kernelcont} of Theorem \ref{Thm:KernelAdmisOp}. Moreover, from the Fourier-domain definition  \eqref{Eq:BesselPot} of $G_s(\cdot,\cdot)$, it can be seen that Property \ref{It:Slowgrowth} also holds. Together, we deduce the admissibility of these kernels. We  remark that the Bessel potential kernels are rotation-invariant as well.
 
 Our final example is a general class of separable shift-invariant kernels of the form 
\begin{equation}\label{Eq:Separable}
{\rm k}(\boldsymbol{x},\boldsymbol{y}) = \prod_{i=1}^d \rho_{\rm L}(x_i-y_i),
\end{equation}
where ${\rm L}:\mathcal{S}'(\mathbb{R})\rightarrow\mathcal{S}'(\mathbb{R})$ is a stable rational operator whose frequency response is  of the form $\widehat{\rm L}(\omega)= \frac{P(\omega)}{Q(\omega)}$, where $P$ and $Q$ are polynomials with no real roots such that ${\rm deg}(P) \geq {\rm deg}(Q)+2$. Since  $\widehat{\rm L}(\omega)$ is real, we conclude that that the tail of $\widehat{\rm L}(\omega)^{-1} = \frac{Q(\omega)}{P(\omega)}$ behaves like $\omega^{-2}$ and  is absolutely integrable  which, together with the Riemann-Lebesgue lemma, implies that $\rho_{\rm L}\in\mathcal{C}_0(\mathbb{R})$. The other conditions of Theorem \ref{Thm:KernelAdmisOp} can be readily shown to be true so  that any separable kernel of the form \eqref{Eq:Separable} is admissible to our theory. 
 
It is worth to mention that one can rotate and dilate any admissible  kernel   by considering an invertible  mixture matrix $\mathbf{A}$  and by defining the transformed kernel as $\mathrm{k}(\mathbf{A} \boldsymbol{x}, \mathbf{A} \boldsymbol{y})$. One readily  verifies that the transformed kernel also satisfies the conditions of Theorem \ref{Thm:KernelAdmisOp} and, hence, is also  admissible.  In Figure \ref{Fig:kernel}, we have plotted the super-exponential and Bessel-potential kernels in  dimension  $d=1$ for different sets of parameters. It can be seen that the width and regularity of these kernels can be adjusted through their parameters. This can be exploited in our framework of learning with multiple kernels to  benefit from this diversity. We shall illustrate this numerically in Section 5.

\begin{figure}[t]
\begin{minipage}{1.0\linewidth}
  \centering
  \centerline{\includegraphics[width=\linewidth]{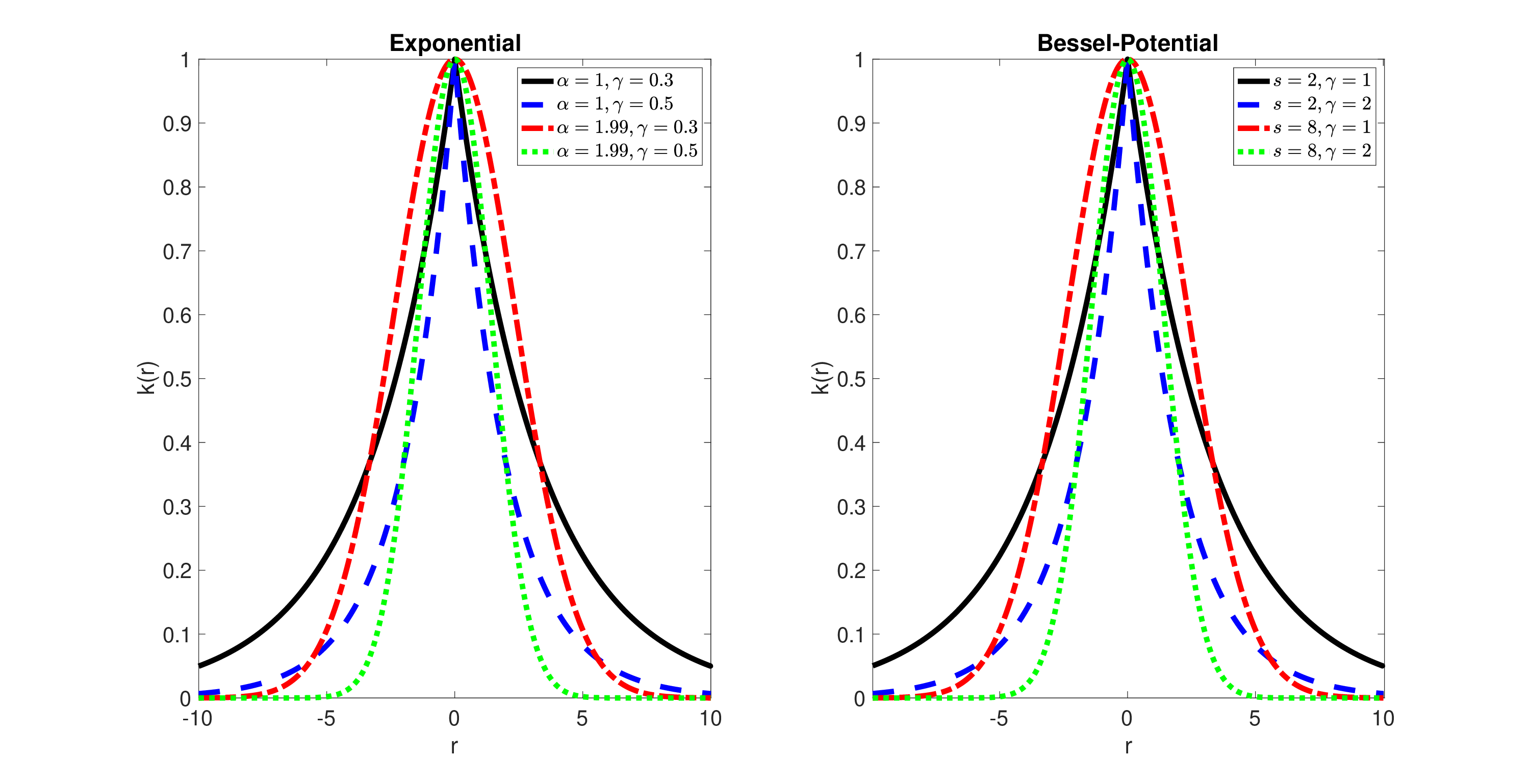}}
  \caption{  Super-exponential kernels $\mathrm{k}_{\alpha}(\boldsymbol{r})=\exp(-\gamma\|\boldsymbol{r}\|_\alpha^\alpha)$ (left) and  Bessel-potential kernels  $G_s(\gamma \boldsymbol{r})$ (right), where $\boldsymbol{r}=(\boldsymbol{x}-\boldsymbol{y})$. The plots are in the special case $d=1$. The parameters ($\alpha \in (0,2)$  and $s>2$) and $\gamma >0$ adjust smoothness and width of the kernel, respectively.}
  \label{Fig:kernel} \medskip
\end{minipage}
\end{figure}

\section{Multiple-Kernel Regression}

In this section, we prove our main result: the representer theorem of multiple-kernel regression with gTV regularization. In effect, the gTV norm will force the learned function to use the fewest active kernels.

\begin{theorem}[Multiple-kernel regression with gTV]\label{Thm:Main}
Given a training dataset that consists of $M$ distinct pairs $(\boldsymbol{x}_m,y_m)$ for $m=1,2,\ldots,M$, we consider the minimization problem
\begin{equation}\label{Pb:MKR}
\min_{\stackrel{ f_n\in\mathcal{M}_{\mathrm{L}_n}(\mathbb{R}^d),}{ f=\sum_{n=1}^N f_n} } \left( \sum_{m=1}^M \mathrm{E}({f}(\boldsymbol{x}_m),{y}_m)+ \lambda \sum_{n=1}^N \|\mathrm{L}_n\{{f}_n\}\|_\mathcal{M}\right), 
\end{equation}
where  $\mathrm{E}(\cdot,{y})$ is a   strictly convex nonnegative function and $\mathrm{L}_n$ is a kernel-admissible operator in the sense of Definition \ref{Def:KernelAdmis} for $n=1,2,\ldots,N$.    Then, the solution set of this problem is nonempty, convex, and weak*-compact. For any of its extreme points  $(f_1,f_2,\ldots,f_N)$, we have the kernel expansions
\begin{equation}\label{Eq:MKRVec}
f_n= \sum_{l=1}^{M_n}  a_{n,l} \mathrm{k}_n(\cdot,\boldsymbol{z}_{n,l}), \quad n=1,2,\ldots,N
\end{equation}
 for its components, where $a_{n,l}\in\mathbb{R}$ are kernel weights, $\boldsymbol{z}_{n,l}\in\mathbb{R}^d$ are adaptive kernel positions, and $\mathrm{k}_n:\mathbb{R}^d\times\mathbb{R}^d\rightarrow\mathbb{R}$ is the shift-invariant kernel  associated to the regularization operator $\mathrm{L}_n$ for $n=1,2,\ldots,N$. Moreover, the number of active kernels is upper-bounded by the number of data points, so that $\sum_{n=1}^N M_n \leq M$.
\end{theorem}
\begin{proof} 
Our proof is divided in three parts. First, we show the existence of a solution. Then, we show  that \eqref{Pb:MKR} is equivalent to a constrained interpolation problem with fixed function values and, from this equivalent form, we deduce the topological properties of the solution set. Finally, we derive the form \eqref{Eq:MKRVec} for the extreme points of the solution set.

Let us denote the data-fidelity and regularization terms of the cost functional  by $H(\cdot)$ and $R(\cdot)$, respectively, so that we have that
\begin{align}
&H(f_1,\ldots,f_N) = \sum_{m=1}^M \mathrm{E}({f}(\boldsymbol{x}_m),{y}_m),  \quad f=\sum_{n=1}^N f_n,\\
&R(f_1,\ldots,f_N) =  \sum_{n=1}^N \|\mathrm{L}_n\{{f}_n\}\|_\mathcal{M}.
\end{align}

\textbf{Part 1: Existence} We apply a standard technique in convex analysis. We show that the cost functional is coercive and weakly lower-semicontinuous  \cite{kurdila2006convex}. This also works when the latter property is replaced by   weak* lower-semicontinuity (see Proposition 8 in \cite{gupta2018continuous}).

The cost functional is a weighted sum of the nonnegative  data-fidelity term $H(f_1,\ldots,f_N)$ and the coercive regularization functional  $R(f_1,\ldots,f_N)$. This ensures its overall  coercivity. 

 The sampling operator is weak*-continuous by assumption. Its composition with a   continuous functional  $\mathrm{E}(\cdot,\boldsymbol{y})$ (that follows from its strict convexity) and summation over $m$ yields a cost functional ${H}({f})$ that is  weak* lower-semicontinuous as well. 

The gTV norms $\|{\rm L}_n\cdot\|_{\mathcal{M}}$ are weak* lower-semicontinuous on $\mathcal{M}_{{\rm L}_n}(\mathbb{R}^d)$.    This implies that the regularization functional is weak* lower-semicontinuous in the product space. Therefore, the overall cost functional ${H}({f_1,f_2,\ldots,f_N} ) + \lambda R(f_1,\ldots,f_N)$ is weak* lower-semicontinuous. Together with the coercivity of the cost functional, this proves the existence of a solution.

\textbf{Part 2: Equivalence to  the Constrained Problem} Considering two solutions $(f_{1,1},\ldots,f_{N,1})$ and $(f_{1,2},\ldots,f_{N,2})$ of the problem, we denote their reconstructing functions by $f_i = \sum_{n=1}^N f_{n,i}$ for $i=1,2$. By contradiction assume that  $f_1(\boldsymbol{x}_m)\neq f_2(\boldsymbol{x}_m)$ for some $m$. Since  $\mathrm{E}(\cdot,y)$ is  a strictly convex function for any $y\in\mathbb{R}$,    we have that
\begin{equation}\label{Ineq:CnvxH}
H(\frac{f_{1,1}+f_{1,2}}{2},\ldots,\frac{f_{N,1}+f_{N,2}}{2}) < \frac{{H}(f_{1,1},\ldots,f_{N,1})+{H}(f_{1,2},\ldots,f_{N,2})}{2}. 
\end{equation}
Similarly, the convexity of $R(\cdot)$ implies the inequality 
\begin{equation}\label{Ineq:CnvxR}
R(\frac{f_{1,1}+f_{1,2}}{2},\ldots,\frac{f_{N,1}+f_{N,2}}{2}) \leq \frac{{R}(f_{1,1},\ldots,f_{N,1})+{R}(f_{1,2},\ldots,f_{N,2})}{2}. 
\end{equation}  
Together, the inequalities \eqref{Ineq:CnvxH} and \eqref{Ineq:CnvxR}   imply that $(\frac{f_{1,1}+f_{1,2}}{2},\ldots,\frac{f_{N,1}+f_{N,2}}{2})$ has a smaller cost than $(f_{1,i},\ldots,f_{N,i})$ for $i=1,2$, which contradicts  their optimality.  Hence,  $ {f}_1(\boldsymbol{x}_m)={f}_2(\boldsymbol{x}_m)={z}_m$ for $m=1,2,\ldots,M$ and one can rewrite the problem as 
\begin{equation}\label{MKRCnst}
\min_{\stackrel{ f_n\in\mathcal{M}_{\mathrm{L}_n}(\mathbb{R}^d),}{ f=\sum_{n=1}^N f_n} }   \sum_{n=1}^N \|\mathrm{L}_n\{{f}_n\}\|_\mathcal{M}, \quad \text{s.t.} \quad  {f}(\boldsymbol{x}_m)={z}_m, \quad m=1,2,\ldots,M.
\end{equation}

\textbf{Part 3: Identifying the Solution Set} Let us define $w_n = \mathrm{L}_n\{f_n\}$ for $n=1,\ldots,N$ and $\nu_m(w_1,\ldots,w_N)=   \sum_{n=1}^N \langle \delta(\cdot-\boldsymbol{x}_m) , {\rm L}_n^{-1}\{w_n\}\rangle = \sum_{n=1}^N f_n(\boldsymbol{x}_m)$ for $m=1,\ldots,M$. We then reformulate \eqref{MKRCnst} as
\begin{equation}\label{VectorFisherJerome}
\min_{ w_1,\ldots,w_N\in\mathcal{M} } \sum_{n=1}^N \|w_n\|_\mathcal{M}, \quad \text{s.t.} \quad  \nu_m(w_1,\ldots,w_N) = z_m , \quad m=1,2,\ldots,M. 
\end{equation}
Now, using the vector-valued Fisher-Jerome theorem (Appendix C), we deduce that the solution set of \eqref{VectorFisherJerome} is convex and weak*-compact   with the extreme points of the form 
$\mathbf{w}=(w_1,\ldots,w_N)$, where $w_n $ takes the form 
\begin{equation}
w_n = \sum_{l=1}^{M_n} a_{n,l} \delta(\cdot-\boldsymbol{z}_{n,l}), 
\end{equation}
for some $a_{n,l}\in\mathbb{R}$ and $\boldsymbol{z}_{n,l}\in\mathbb{R}^d$. Moreover, the total number of Diracs in   $\mathbf{w}$ is upper-bounded by $M$. This implies that the solution set of  \eqref{MKRCnst} (and, consequently, the one of \eqref{Pb:MKR}) is a convex and weak*-compact set  due to the linearity and isomorphism of $\mathrm{L}_n$. Correspondingly, the extreme points of the original problem \eqref{Pb:MKR} take  the form of $(f_1,f_2,\ldots,f_N)$, where $f_n = \mathrm{L}_n^{-1}\{w_n\}$ has a kernel expansion with $M_n$ kernels at adaptive positions subject to the constraint  $\sum_{n=1}^N M_n \leq M$. 
\end{proof}

 The practical outcome of Theorem \ref{Thm:Main} is that any extreme point of \eqref{Pb:MKR} maps into a solution of the form 

 \begin{equation}\label{Eq:MKRSol}
f(\cdot) = \sum_{n=1}^{N}\sum_{l=1}^{M_n}  a_{n,l} \mathrm{k}_n(\cdot,\boldsymbol{z}_{n,l})
 \end{equation}
 for the learned function.  The solution form \eqref{Eq:MKRSol} has the following important properties:
\begin{itemize}
\item The number of active kernels is upper-bounded by the number of samples $M$.  This justifies the use of multiple kernels since the flexibility of the model will be increased while the problem  remains well-posed. 
\item The gTV norm enforces   an $\ell_1$ penalty on the kernel coefficients.  Practically, this will result in an $\ell_1$-minimization problem that is reminiscent of   the generalized LASSO.  
\item The kernel positions are adaptive and will be chosen such that the solution becomes sparse.  In other words, the adaptiveness of the kernel positions, together with the $\ell_1$ regularization on the kernel coefficients, favors solutions with a small number of nonzero terms in   the expansion  \eqref{Pb:MKR}.
\end{itemize}

To conclude this section, let us mention that the existence of the kernel locations $\mathbf{z}_{n,l}$ in \eqref{Eq:MKRSol} is guaranteed  by our representer theorem. However, unlike in RKHS methods, these locations do not necessarily coincide with the data points. The adaptiveness comes from the fact that the kernel positions become part of the reduced finite-dimensional optimization problem (see \eqref{Pb:gTVdiscreteLASSO} for the single-kernel scenario). Hence, an optimization scheme is required in order to ``learn'' these unknown parameters along with the kernel weights.

\section{Discussion and Illustration} 
 In this section, we provide some further discussions together with a  numerical example that illustrates  important aspects of our framework. 
\subsection{Optimization Scheme}  
Finding the kernel positions in general can be very challenging. Once the positions are fixed, one can find the kernel weights efficiently using  classical $\ell_1$-minimization techniques  \cite{beck2009fast,daubechies2010iteratively,figueiredo2007gradient}.  In low dimensions, one can use grid-based algorithms and reach  a solution where the positions of the kernels are quantified   \cite{debarre2019b,gupta2018continuous}. It is also possible to adapt the algorithms developed for finding Dirac locations in  super resolution in order to find the kernel positions  \cite{bredies2013inverse,candes2013super,denoyelle2019sliding,fernandez2016super}. However, for high-dimensional problems, this is an open numerical challenge and requires further considerations. A  possible avenue of research would be to use  first-order primal-dual  splitting methods for convex-nonconvex problems  \cite{valkonen2019first} and take advantage of the convexity of the problem with respect to the kernel weights.

\subsection{Numerical Example}\label{Sec:Numerical}
   In this section, we provide a numerical example in the case $d=1$. We would like to emphasize  that the computational aspects of our framework ({\it e.g.}, the derivation of  efficient algorithms in high dimensions) is left to future works. The sole purpose of our example is to illustrate the use of  Theorem \ref{Thm:Main} and highlight two important features, namely, adaptivity and sparsity.    

In our example, we compare the performance of five kernel estimators: 
\begin{enumerate}
\item {\bf RKHS $\bf L_2$:} RKHS regularization  \eqref{Pb:DiscreteL2}.  
\item {\bf RKHS $\bf L_1$:}  Generalized LASSO \eqref{Pb:GeneralizedLASSO}.  
\item {\bf SimpleMKL:} Multiple-kernel learning (MKL) using the SimpleMKL algorithm \cite{rakotomamonjy2008simplemkl}.
\item {\bf Single gTV:} Single-kernel   gTV regularized learning \eqref{Pb:gTVdiscreteLASSO}.
\item {\bf Multi gTV:} Learning with multiple kernels and gTV regularization \eqref{Pb:MKR}.
\end{enumerate}

To avoid the difficulty of optimizing over the data centers in the gTV-based methods,  which would result in a nonconvex problem, we use a convex proxy in which a redundant set of centers is placed on a grid and  the excess ones  are suppressed with the help of $\ell_1$-minimization. With this grid-based approach,  the search for the kernel positions is reduced to   a large-scale $\ell_1$-minimization problem for which robust algorithms  are known to exist---specifically, we  have used FISTA \cite{beck2009fast} in our example. This scheme will obviously only work when the input dimension is very low, such as $d=1$ in the present example.

 We consider the reconstruction of a function from its noisy samples in two scenarios: full data {\it versus} missing data.  The results are depicted in Figure \ref{CompareNoisy}, while we refer to Appendix \ref{App:Detail} for the full implementation details.  As we can see in Figure \ref{SubFig:Fulldata}, due to the presence of a non-smooth region in the target function, the single-kernel methods are forced to use narrow kernels with a small width which creates undesirable oscillations in the smoother regions. By contrast, our multi-kernel scheme uses both narrow and wide kernels,  hence  providing the reconstruction with the least fluctuation. In the presence of missing data, we observe in Figure \ref{SubFig:MissData} that the reconstructed function of RKHS-based methods, exhibit an undesirable dip. This is due to the fact that, in the RKHS-based methods, the kernel functions are located on the data points and their width is too short to fill the gap in the data. By contrast, the kernel locations are adaptive in our scheme, which yields a decent reconstruction in this case as well. Finally, we have plotted the 100 largest kernel coefficients  of each expansion in the full-data experiment in Figure \ref{CoeffNoisy}. This plot highlights that the gTV-based methods are providing the sparsest representation for the target function, as expected.

 
The above visual observations are also supported quantitatively in Table 1, where we report the mean-squared error (MSE) error and   sparsity   (number of coefficients that are larger than one tenth of the maximum coefficient)  of each method in the two scenarios.
%

 \begin{figure}[t]
  \begin{subfigure}{.49\textwidth}
  \centering
  \centerline{\includegraphics[width=\linewidth]{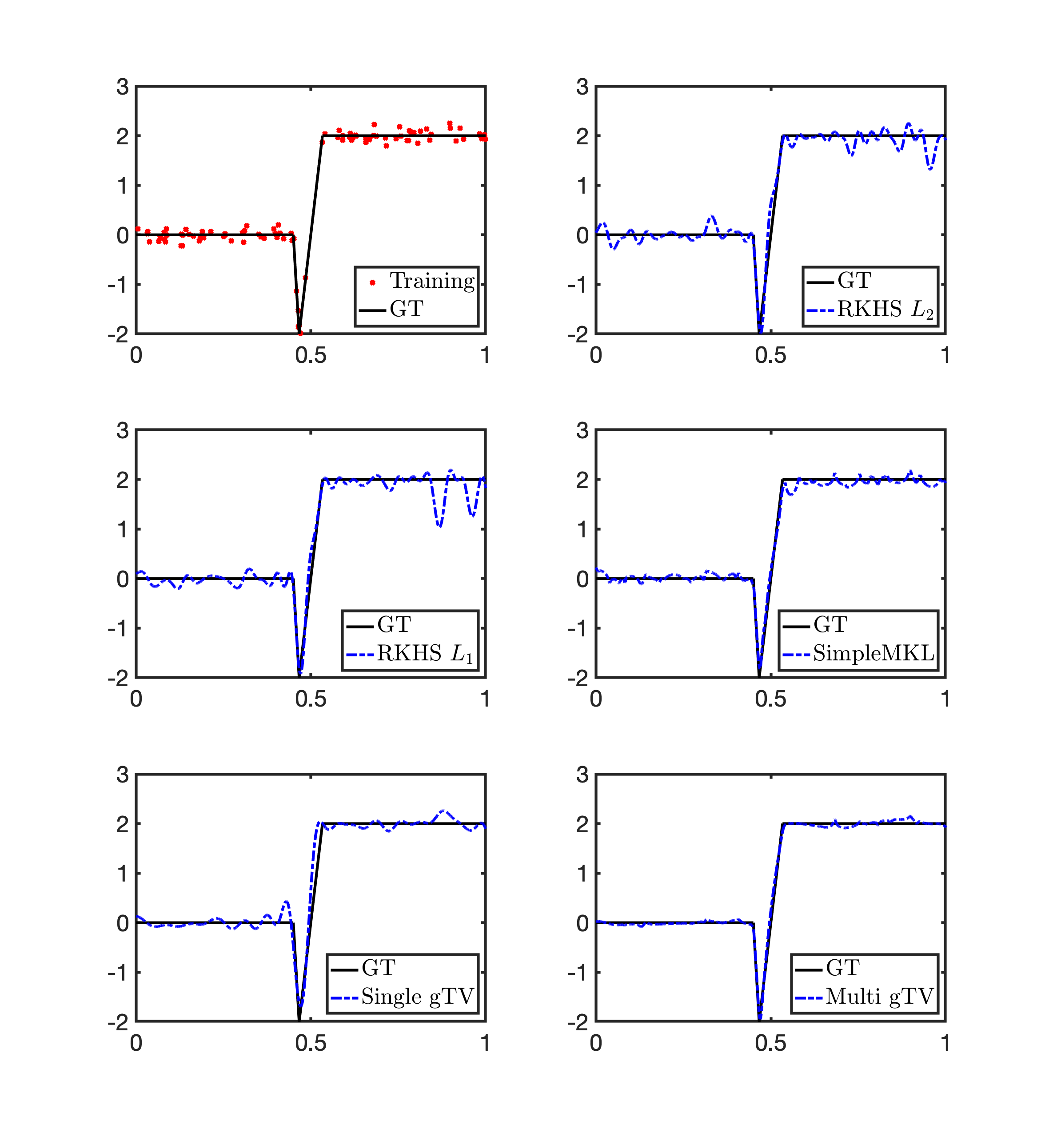}}
  \caption{Full data}\label{SubFig:Fulldata}
\end{subfigure}
  \begin{subfigure}{.49\textwidth}
  \centering
  \centerline{\includegraphics[width=\linewidth]{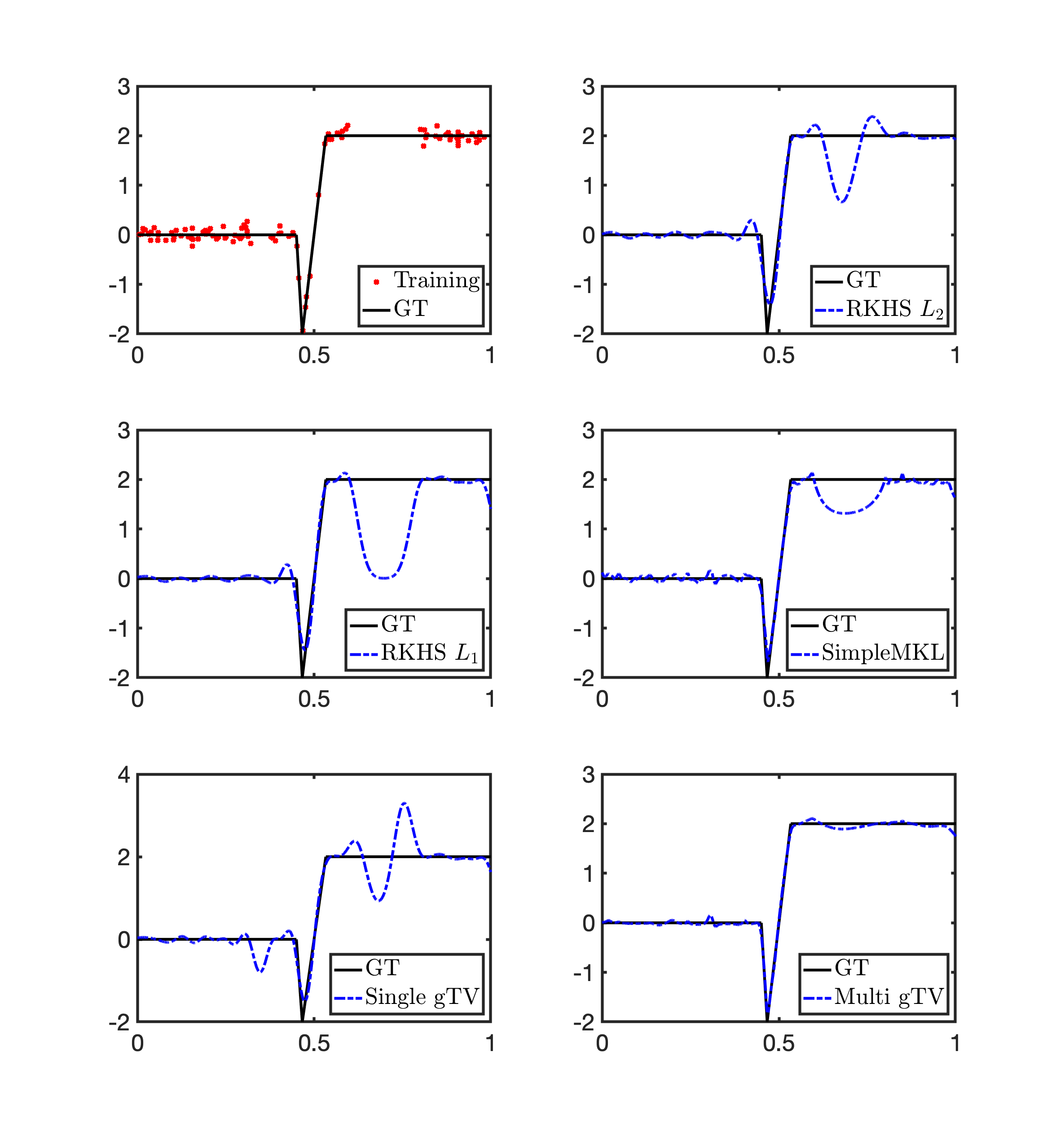}}
  \caption{Missing data}\label{SubFig:MissData}
\end{subfigure}

  \caption{ Performance of  the kernel estimators in two scenarios: full data (Left) or missing data (Right). Solid line: ground-truth (GT) function. Dash-dotted line: reconstructed functions. Dots: noisy data points.}\label{CompareNoisy} \medskip
\end{figure}
%
 
  \begin{figure}[t]
\begin{minipage}{1.0\linewidth}
  \centering
  \centerline{\includegraphics[width=.7\linewidth]{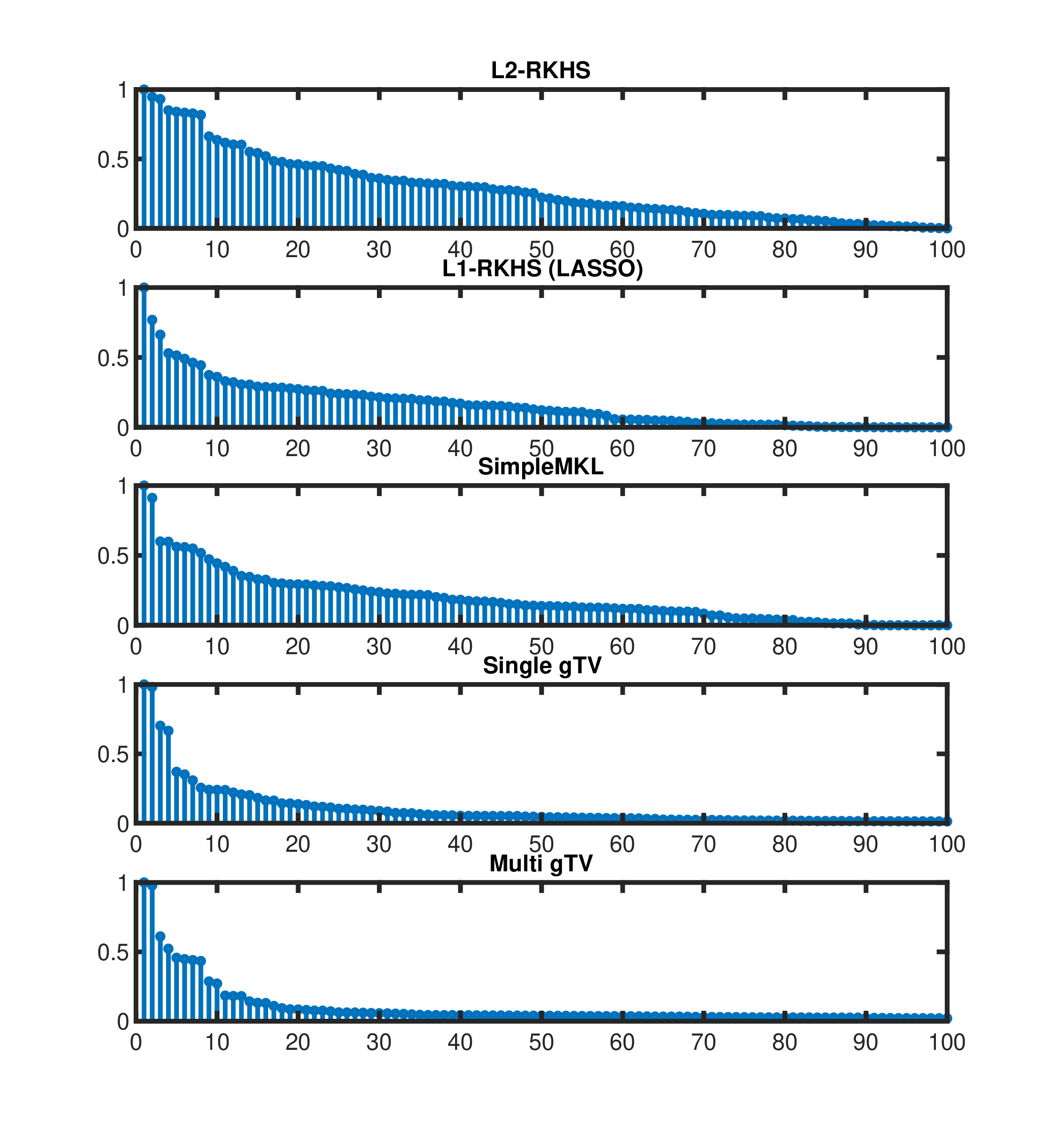}}
  \caption{ 100 largest coefficients of each expansion in the full-data  case.}\label{CoeffNoisy} \medskip
\end{minipage}
\end{figure}

\begin{table}[!t]
	\renewcommand{\arraystretch}{1.3}
	\centering
	\caption{MSE and sparsity of the  kernel estimators. The results are averaged over 10 runs. }
	\begin{scriptsize}
		\begin{tabular}{c c| c c c c c}
			\hline
			\hline
			 Quantity &Dataset&  L2-RKHS& L1-RKHS & SimpleMKL & Single gTV & Multi gTV \\
		 	\hline
			\multirow{2}{*}{Sparsity} & Full data & 64.7  &  44.1 & 54.4  & 32.5 & {\bf 20.0 }\\			 
			& Missing data & 66.1 & 39.3 & 56.0 & 32.9 & {\bf 31.1} \\
			\hline 
			\multirow{2}{*}{MSE (dB)}  & Full data &  -17.2  &    -16.1 &  -15.2 &   -16.7 & {\bf  -18.1 }\\			 
			& Missing data& -2.6  &  -2.7  & -10.9  &  -3.9  & {\bf -17.3 } \\
			\hline
			\hline 
		\end{tabular}
	\end{scriptsize}
\end{table}
 
\subsection{Uniqueness of the Solution}
An interesting question is to explore the cases where perfect recovery is theoretically guaranteed. This is an open research topic on its own for which a rich literature exists  \cite{duval2015exact,denoyelle2017support,de2012exact}. Nevertheless, we analyze in Proposition \ref{Prop:Uniqueness}, a very simple scenario  for which we can prove uniqueness. 
\begin{proposition}\label{Prop:Uniqueness}
Let ${\rm k}_1,\ldots,{\rm k}_N$  be a collection of $N$ symmetric admissible  kernels that are normalized so that  ${\rm k}_n(\boldsymbol{x},\boldsymbol{x})= 1$ for $n=1,\ldots,N$ and $\boldsymbol{x}\in\mathbb{R}^d$. Consider the  minimization 
\begin{equation}\label{Eq:example}
\min_{\stackrel{ f_n\in\mathcal{M}_{\mathrm{L}_n}(\mathbb{R}^d),}{ f=\sum_{n=1}^N f_n} }   \sum_{n=1}^N \|\mathrm{L}_n\{{f}_n\}\|_\mathcal{M}, \quad \text{s.t.} \quad  {f}(\boldsymbol{x}_m)=f_0(\boldsymbol{x}_m), \quad m=1,\ldots,M,
\end{equation}
where $f_0(\cdot)= a_0 {\rm k}_{n_0} (\cdot,\boldsymbol{z}_0)$ for some $n_0\in\{1,\ldots,N\}$, $a_0\in\mathbb{R}$, and $\boldsymbol{z}_0\in\mathbb{R}^d$. Assume that the set of data points $\{\boldsymbol{x}_1,\ldots,\boldsymbol{x}_M\}$ contains $\boldsymbol{z}_0$ and is such that the $M$ by $N$ matrix ${\bf K} = [{\rm k}_n(\boldsymbol{x}_m,\boldsymbol{z}_0)]$ has full column rank. Then, $f_0$ is the unique solution of \eqref{Eq:example}. 
\end{proposition}
\begin{proof}
From the proof of Theorem \ref{Thm:Main}, we know that the solution set of \eqref{Eq:example} is the convex hull of functions of the form \eqref{Eq:MKRSol}. We now show that the solution set has only one extreme point  (that is $f_0$), which is equivalent to the solution being unique. 

Let $f$ be an extreme point of the solution set of \eqref{Eq:example} whose form is given in \eqref{Eq:MKRSol}.  Since $\boldsymbol{z}_0$ is among the data points, we deduce that 
$$a_0 = a_0 {\rm k}_{n_0} (\boldsymbol{z}_0,\boldsymbol{z}_0) = f_0 (\boldsymbol{z}_0)= f(\boldsymbol{z}_0) = \sum_{n=1}^{N}\sum_{l=1}^{M_n}  a_{n,l} \mathrm{k}_n(\boldsymbol{z}_0,\boldsymbol{z}_{n,l}).$$
Hence, by using the triangle inequality, we obtain that 
$$ |a_0| = \left| \sum_{n=1}^{N}\sum_{l=1}^{M_n}  a_{n,l} \mathrm{k}_n(\boldsymbol{z}_0,\boldsymbol{z}_{n,l})\right| \leq \sum_{n=1}^{N}\sum_{l=1}^{M_n}  |a_{n,l}| |\mathrm{k}_n(\boldsymbol{z}_0,\boldsymbol{z}_{n,l})|\leq \sum_{n=1}^{N}\sum_{l=1}^{M_n}  |a_{n,l}|,$$
where the last inequality comes from the fact that, for any positive-definite kernel ${\rm k}$, we have that ${\rm k}(\boldsymbol{x},\boldsymbol{y})^2 \leq {\rm k}(\boldsymbol{x},\boldsymbol{x}){\rm k}(\boldsymbol{y},\boldsymbol{y}) = 1$. Note that the positive-definiteness here is guaranteed by Corollary \ref{Corollary}. This Cauchy-Schwartz-type inequality is saturated if and only if $\boldsymbol{x}=\boldsymbol{y}$.  Together with the optimality of   $f$, we deduce that  $\boldsymbol{z}_0=\boldsymbol{z}_{n,l} $ for all $n=1,\ldots, N$ and $l=1,\ldots,M_n$. Hence, we can rewrite the constraints  as 
$$ \sum_{n=1}^N \tilde{a}_n {\rm k}_n(\boldsymbol{x}_m,\boldsymbol{z}_0)= a_0 {\rm k}_{n_0} (\boldsymbol{x}_m,\boldsymbol{z}_0) , \quad m=1,\ldots,M,$$
where $\tilde{a}_n = \sum_{l=1}^{M_n} a_{n,l}$. In matrix form, this becomes ${\bf K} \tilde{\boldsymbol{a}} = {\bf K} a_0 \mathbf{e}_{n_0}$, where $\mathbf{e}_{n_0} \in\mathbb{R}^N$ is the $n_0$th element of the canonical basis of $\mathbb{R}^N$. Finally, by using the full column rank assumption, we deduce that $\tilde{\boldsymbol{a}} = a_0 \mathbf{e}_{n_0}$, which completes the proof.
\end{proof}

\section{Conclusion}
In this paper, we have provided a theoretical foundation for  multiple-kernel regression with gTV regularization.  We  have studied the Banach structure of our search space and     identified the class of  kernel functions that are admissible. Then, we have derived a representer theorem that shows that the    learned function  can be written as a linear combination of   kernels   with adaptive centers.   Our representer theorem also provides an upper bound to the number of active elements, which allows us to use as many kernels as convenient.  We have illustrated numerically the effect of using multiple kernels with a sparsity constraint. Further research directions could be the development of efficient methods in high dimensions to approximate the kernel positions and an extension of the current theory to make Gaussian kernels admissible. 
 
\appendix \section{Proof of Theorem \ref{Thm:NativeSpace} }
\begin{proof}

\ref{It:MLBanach} The linearity and invertibility of $\mathrm{L}$   implies that the native space together with the gTV norm is a {\it bona fide} Banach space. 

\ref{It:MLM} The restriction of $\mathrm{L}$ over its native space is injective (inherited from $\mathrm{L}$) and is continuous  due to the definition of the  gTV norm. For all ${w}\in\mathcal{M}(\mathbb{R}^d )$, the relation $\mathrm{L}\{ \mathrm{L}^{-1}\{{w}\} \} = {w}$ implies that it is surjective as well and that its inverse is the restriction of $\mathrm{L}^{-1}$ over $\mathcal{M}(\mathbb{R}^d )$ which continuously maps $\mathcal{M}(\mathbb{R}^d)\rightarrow\mathcal{M}_{\mathrm{L}}(\mathbb{R}^d)$. This ensures that $\mathrm{L}:\mathcal{M}_{\mathrm{L}}(\mathbb{R}^d) \rightarrow \mathcal{M}(\mathbb{R}^d)$ is an isomorphism.

\ref{It:Adjoint} The isomorphism of Part \ref{It:MLM} implies the existence of the adjoint operator over $\left(\mathcal{M}(\mathbb{R}^d )\right)' $. By restricting the adjoint operator to   $\mathcal{C}_0(\mathbb{R}^d )$, we obtain the  operator $\mathrm{L}^*:\mathcal{C}_0(\mathbb{R}^d) \rightarrow \mathcal{C}_{\mathrm{L}}(\mathbb{R}^d ) $, where  the space $ \mathcal{C}_{\mathrm{L}}(\mathbb{R}^d )$ is the image of $\mathrm{L}^*$ over $\mathcal{C}_0(\mathbb{R}^d )$.   This space,   equipped with the norm  $\|f\|_{\mathcal{C}_\mathrm{L}}\eqdef\|\mathrm{L}^{-1*}\{f\}\|_\infty$, is a Banach space due to the  linearity and invertibility of $\mathrm{L}^{-1*}$. 

\ref{It:Predual} Similarly to Part \ref{It:MLM}, we readily  verify that  the adjoint operator $\mathrm{L}^*:\mathcal{C}_0(\mathbb{R}^d)\rightarrow\mathcal{C}_{\mathrm{L}}(\mathbb{R}^d)$ is indeed an isomorphism. Therefore, the double-adjoint operator is the isomorphism $\mathrm{L}^{**}: (\mathcal{C}_\mathrm{L}(\mathbb{R}^d ))' \rightarrow  \mathcal{M} (\mathbb{R}^d )$. Consequently, the domains of $\mathrm{L}$ and $\mathrm{L}^{**}$ must be equal, which implies that  $\mathcal{C}_\mathrm{L}(\mathbb{R}^d )$ is the predual of the native space.

\ref{It:Embedding} First, we  show that the operator $\mathrm{L}$ is closed over the space of Schwartz functions.  It is known that the impulse response of $\mathrm{L}^*:\mathcal{S}\rightarrow\mathcal{S}$ is the flipped version of the one of $\mathrm{L}$ \cite{unser2014introduction}. In other words, the application of $\mathrm{L}^*$ on a Schwartz function can be expressed by  
\begin{equation}\label{SchwartzThmAdjoint}
\forall \varphi \in \mathcal{S}(\mathbb{R}^d):  \mathrm{L}^*\{\varphi \}(\cdot)= \int_{\mathbb{R}^d}   h(\boldsymbol{x}-\cdot)\varphi (\boldsymbol{x}) \mathrm{d}\boldsymbol{x},
\end{equation}
where $h\in\mathcal{S}'(\mathbb{R}^d)$ is the impulse response of $\mathrm{L}$, described  in 
\eqref{SchwartzKernelTheorem}. By the change of variable $\boldsymbol{y}=(-\boldsymbol{x})$, one  verifies that,  for any $\varphi\in \mathcal{S}(\mathbb{R}^d)$, we have that 
\begin{equation}\label{Eq:LLstar}
\mathrm{L}\{\varphi \} = \mathrm{L}^*\{ \varphi^{\vee}\}^\vee,
\end{equation}
where $\varphi^\vee$ is the flipped version of $\varphi\in\mathcal{S}(\mathbb{R}^d)$ with $\varphi^\vee(\boldsymbol{x}) = \varphi(-\boldsymbol{x})$ for all $\boldsymbol{x}\in\mathbb{R}^d$. In effect, \eqref{Eq:LLstar} shows that  $\mathrm{L}\{ \varphi\} \in\mathcal{S}(\mathbb{R}^d)$ for any $\varphi\in \mathcal{S}(\mathbb{R}^d)$. 

Now, from the inclusions $\mathrm{L}\{\mathcal{S}(\mathbb{R}^d)\} \subseteq \mathcal{S}(\mathbb{R}^d)$ and  $\mathcal{S}(\mathbb{R}^d) \subseteq \mathcal{M}(\mathbb{R}^d)$, we deduce that  $\mathcal{S}(\mathbb{R}^d )  \subseteq \mathcal{M}_{\mathrm{L}}(\mathbb{R}^d ) $.   Moreover, $\mathcal{M}_{\mathrm{L}}(\mathbb{R}^d ) \subseteq \mathcal{S}'(\mathbb{R}^d)$  by Definition \ref{Def:NativeSpace}. This verifies the inclusion $\mathcal{S}(\mathbb{R}^d) \subseteq \mathcal{M}(\mathbb{R}^d) \subseteq \mathcal{S}'(\mathbb{R}^d)$. To complete the proof, we need to show that the identity operators $id_1:\mathcal{S}(\mathbb{R}^d) \rightarrow \mathcal{M}(\mathbb{R}^d)$ and $id_2: \mathcal{M}(\mathbb{R}^d) \rightarrow \mathcal{S}'(\mathbb{R}^d)$ are continuous.

For a converging sequence of Schwartz functions $\varphi_n\stackrel{\mathcal{S}}{ \rightarrow} \varphi$, the continuity of $\mathrm{L}$ implies that $\mathrm{L}\{\varphi_n\}\stackrel{\mathcal{S}}{ \rightarrow} \mathrm{L}\{\varphi\}$. Since  $\mathcal{S}(\mathbb{R}^d)$ is continuously embedded in $\mathcal{M}(\mathbb{R}^d)$, we have that $\mathrm{L}\{\varphi_n\}\stackrel{\mathcal{M}}{ \rightarrow}\mathrm{L}\{\varphi\}$ and, consequently, that  $\varphi_n \stackrel{\mathcal{M}_{\mathrm{L}}}{ \rightarrow} \varphi$. This proves that the embedding is continuous, which is denoted by $\mathcal{S}(\mathbb{R}^d )  {\hookrightarrow} \mathcal{M}_{\mathrm{L}}(\mathbb{R}^d )$. Moreover, since the space $\mathcal{M}(\mathbb{R}^d)$ is continuously embedded in $\mathcal{S}'(\mathbb{R}^d)$, the convergence $\mathrm{L}\{\varphi_n\}\stackrel{\mathcal{M}}{ \rightarrow}\mathrm{L}\{\varphi\}$ implies that $\mathrm{L}\{\varphi_n\}\stackrel{\mathcal{S}'}{ \rightarrow}\mathrm{L}\{\varphi\}$.   This proves that    $\mathcal{M}_{\mathrm{L}}(\mathbb{R}^d ) \stackrel{d.}{\hookrightarrow}  \mathcal{S}'(\mathbb{R}^d )$. The latter continuous embedding  is also dense due to the denseness of $\mathcal{S}(\mathbb{R}^d)$ in  $\mathcal{S}'(\mathbb{R}^d)$ and  the inclusion  $\mathcal{S}(\mathbb{R}^d) \subseteq \mathcal{M}_{\mathrm{L}}(\mathbb{R}^d )$.
\end{proof}

 \section{Proof of Theorem \ref{Thm:KernelAdmisOp}}
\begin{proof} 
Assume that $\mathrm{L}$ is a kernel-admissible operator. The weak*-continuity of the sampling functional  implies that the shifted Dirac impulses $\delta(\cdot-\boldsymbol{x}_0)$ should be included in the predual space $\mathcal{C}_{\mathrm{L}}(\mathbb{R}^d)$. Therefore, $\mathrm{L}^{-1*}\{\delta(\cdot-\boldsymbol{x}_0)\}$ should be   in $\mathcal{C}_0(\mathbb{R}^d)$. Since, the Green's functions of $\mathrm{L}$ and $\mathrm{L}^*$ are flipped version of each other, we deduce that $\rho_{\mathrm{L}}= \mathrm{L}^{-1}\{\delta(\cdot-\boldsymbol{x}_0)\}\in \mathcal{C}_0(\mathbb{R}^d)$.  For the second property, we recall that the continuity of $\mathrm{L}^{-1}:\mathcal{S}'(\mathbb{R}^d)\rightarrow\mathcal{S}'(\mathbb{R}^d)$ implies the smoothness and slow growth of  the Fourier transform of its frequency response. Hence, $\widehat{\rho_\mathrm{L}}(\omega)$ is smooth and slowly growing.  Similarly, the continuity of $\mathrm{L}$ implies that  $\frac{1}{\widehat{\rho_\mathrm{L}}(\omega)}$ is   a smooth and slowly growing function as well. Thus,  $\widehat{\rho_\mathrm{L}}(\omega)$  is non-vanishing and heavy-tailed. 

For the converse, assume that the   function $\rho$ satisfies Properties \ref{It:kernelcont}  and \ref{It:Slowgrowth} in Theorem \ref{Thm:KernelAdmisOp}.  First, note that, if $f,g:\mathbb{R}^d\rightarrow\mathbb{R}$ are smooth and slowly growing functions and, moreover, $g$ is nonzero and heavy-tailed, then
\begin{equation}
 { \partial  \over \partial x_{i}} \left(\frac{f}{g} \right)=\frac{ {\partial f \over \partial x_{i}}   g  -  {\partial g  \over \partial x_{i} }  f }{g ^2}
\end{equation}
is a quotient whose  numerator is  a smooth and slowly growing function and whose denominator $g^2$ is a  nonzero, heavy-tailed, smooth, and slowly growing function. Hence, the quotient itself is a smooth function whose  growth  is bounded  by a polynomial. Based on this observation, one can deduce from induction that all the arbitrary-order   derivatives of  $ \frac{1}{\widehat{\rho}(\boldsymbol{\omega})}$ can be expressed by a quotient with a slowly growing nominator and a heavy-tailed denominator. This shows that  $ \frac{1}{\widehat{\rho}(\boldsymbol{\omega})}$ is a smooth and  slowly growing function as well. These properties ensure the existence of   continuous LSI operators $\
\mathrm{L},\tilde{\mathrm{L}}:\mathcal{S}'(\mathbb{R}^d)\rightarrow \mathcal{S}'(\mathbb{R}^d)$ with 
the frequency responses $\frac{1}{\widehat{\rho}(\omega)}$ and $ \widehat{\rho}(\omega)$, 
respectively. The one-to-one correspondence between an operator and its frequency response then 
yields that $\tilde{\mathrm{L}} = \mathrm{L}^{-1}$, from which we conclude that $\mathrm{L}$ is an isomorphism 
over $\mathcal{S}'(\mathbb{R}^d)$. Moreover, due to  Property \ref{It:kernelcont}, we know that the Green's function of $\mathrm{L}$  is in $\mathcal{C}_0(\mathbb{R}^d)$. Hence, the Green's function of $\mathrm{L}^*$ is also in $\mathcal{C}_0(\mathbb{R}^d)$ so that,   for any $\boldsymbol{x}_0\in\mathbb{R}^d$, we have that 
\begin{equation}
\mathrm{L}^{-1*}\{\delta(\cdot-\boldsymbol{x}_0)\} = \mathrm{L}^{-1*}\{\delta\}(\cdot-\boldsymbol{x}_0) \in \mathcal{C}_0(\mathbb{R}^d).
\end{equation}  
In other words, $\delta(\cdot-\boldsymbol{x}_0) \in \mathrm{L}^*(\mathcal{C}_0(\mathbb{R}^d)) = \mathcal{C}_{\mathrm{L}}(\mathbb{R}^d)$, which shows that  the sampling functionals are weak*-continuous. 
\end{proof}

 \section{Vector-Valued Fisher-Jerome Theorem} 
 
Here, we propose and prove a generalization of the Fisher-Jerome theorem \cite{fisher1975spline} for a vector of bounded Radon measures.  The result is not deducible from the original theorem, but its proof is an adaptation of the scalar case (Theorem 7 in \cite{unser2017splines}). We denote the space of bounded Radon vector measures $(w_1,\ldots,w_N)$ by $\mathcal{M}(\mathbb{R}^d;\mathbb{R}^N)$, where each component $w_n \in \mathcal{M}(\mathbb{R}^d)$ is a bounded Radon measure. The total-variation norm of the vector $\mathbf{w}= (w_1,\ldots,w_N) \in \mathcal{M}(\mathbb{R}^d;\mathbb{R}^N)$ is defined by $\|\mathbf{w}\|_{\mathcal{M}}= \sum_{n=1}^N \|w_n\|_{\mathcal{M}}$.
\begin{theorem}[Vector-valued Fisher-Jerome]\label{GGF}
Let $\mathcal{B}= \mathcal{M}(\mathbb{R}^d;\mathbb{R}^N) \bigoplus \mathcal{N}$, where $\mathcal{N}$ is an $N_0$-dimensional normed space and assume that $\mathrm{F}:\mathcal{B}\rightarrow \mathbb{R}^M$ is a linear and weak*-continuous functional ($M\geq N_0$)   such that 
\begin{equation}\label{condGGF}
\exists B>0: \quad \forall \mathbf{p} \in \mathcal{N}\backslash \{\boldsymbol{0}\}, \quad B \leq \frac{\|\mathrm{F}(0,\mathbf{p})\|_2}{\|\mathbf{p}\|_\mathcal{N}}
\end{equation}
and that the minimization problem 
\begin{equation}\label{probGGF}
\mathcal{V}=\argmin_{ (\mathbf{w},\mathbf{p})\in\mathcal{B}}  \|\mathbf{w} \|_\mathcal{M}  \quad s.t.\quad \mathrm{F}(\mathbf{w},\mathbf{p}) \in \mathcal{C} 
\end{equation}
is  feasible  for a convex and compact set $\mathcal{C}\subseteq \mathbb{R}^M$.  Then, $\mathcal{V}$ is a   nonempty, convex, weak*-compact subset of $\mathcal{B}$ while the components of its  extreme points  $(w_1,w_2,\ldots,w_{N},\mathbf{p})$ are all of the form  
\begin{equation}\label{epform}
w_n = \sum_{l=1}^{M_n} a_{n,l} \delta(\cdot-\boldsymbol{z}_{n,l}), \quad n=1,2,\ldots,N,
\end{equation}
where $a_{n,l} \in \mathbb{R}$ and $\boldsymbol{z}_{n,l} \in \mathbb{R}^d $. Moreover, $\sum_{n=1}^{N} M_n \leq M$ and the minimum $\mathcal{M}$-norm obtained for the problem is equal to $ \sum_{n=1}^{N} \sum_{l=1}^{M_n} |a_{n,l}|$.
\end{theorem}
\begin{proof}
 The proof is in two parts. First, we  show that the solution set is nonempty, weak*-compact, and convex. Then, we  explore the form of its extreme points to complete the theorem. 

{\bf Structure of the Solution Set} Consider a point in the feasible set and denote it by $(\mathbf{w}_0,\mathbf{p}_0)$. Then,  \eqref{probGGF} is equivalent to the minimization 
\begin{equation}\label{probGGF2}
\mathcal{V}=\argmin_{ (\mathbf{w},\mathbf{p})\in\mathcal{B}}  \|\mathbf{w} \|_\mathcal{M}  \quad s.t.\quad \mathrm{F}(\mathbf{w},\mathbf{p}) \in \mathcal{C},\|\mathbf{w}\|_{\mathcal{M}}< \|\mathbf{w}_0\|_{\mathcal{M}}. 
\end{equation}
 Since $\mathcal{C}$ is compact,  $A=\max_{\boldsymbol{x}\in \mathcal{C}} \|\boldsymbol{x}\|_2$ will be  a finite constant. Due to the linearity of $F(\cdot,\cdot)$ and due to the triangle inequality, for all $(\mathbf{w},\mathbf{p})$ in the feasible set  we have that
\begin{equation}\label{Eq:PBound}
 \|\mathbf{p}\|_\mathcal{N}   \leq \frac{1}{B}\|F(0,\mathbf{p})\|_2 = \frac{1}{B}\|F(\mathbf{w},\mathbf{p})-F(\mathbf{w},0)\|_2   \leq \frac{1}{B} (A+ \|w_0\|_{\mathcal{M}}). 
\end{equation}
 The conclusion is that the feasible set of \eqref{probGGF2} is bounded. It is also weak*-closed due to the weak*-continuity of $F(\cdot,\cdot)$ and the closedness of $\mathcal{C}$. Hence, it is weak*-compact due to the Banach-Alaoglu theorem  \cite[Theorem 3.15]{rudin1991functional}. The conclusion is that \eqref{probGGF} is equivalent to the minimization of  a weak*-continuous functional over a weak*-compact domain. Moreover, due to the generalized Weierstrass theorem \cite[Theorem 7.3.1]{kurdila2006convex}, its solution set is nonempty. Denote the optimal cost of \eqref{probGGF}  by $\beta$.  Now, note that the feasible set of \eqref{probGGF} is the preimage of the linear continuous functional $F$ over the convex  set $\mathcal{C}$. So, one can rewrite the solution set $\mathcal{V}$ as 
\begin{equation}\label{Eq:SolsetInt}
\mathcal{V} = F^{-1}(\mathcal{C}) \cap \{\mathbf{w}\in \mathcal{M}(\mathbb{R}^d;\mathbb{R}^N) : \|\mathbf{w}\|_{\mathcal{M}} = \beta \}.
\end{equation}
This implies that $\mathcal{V}$ is bounded (due to \eqref{Eq:PBound}), weak*-closed, and convex (the intersection of two weak*-closed and convex set). Hence, it is also weak*-compact. Using the Krein-Milman theorem (\cite{rudin1991functional}, Theorem 3.23), we deduce that $\mathcal{V}$ is the convex hull of its extreme points. 

{\bf Form of the Extreme Points} Consider an arbitrary extreme point of $\mathcal{V}$ such as $(\mathbf{w},\mathbf{p})$, where $\mathbf{w}=(w_1,w_2,\ldots,w_{N})$. We   show that it is not possible to have disjoint Borelian sets $E_{n,l}\subseteq \mathbb{R}^d$  such that $\langle {w}_n,\mathbbm{1}_{E_{n,l}} \rangle \neq 0$, where $ n=1,2,\ldots,N $ and $l=1,2,\ldots ,M_n$ with $\sum_{n=1}^{N} M_n\geq M+1$. We prove the result by contradiction.  Assume   such  disjoint sets exist. Define $v_{n,l}= w_n \mathbbm{1}_{E_{n,l}}$, $\mathbf{v}_{n,l}=\boldsymbol{e}_n v_{n,l}$,  $\overline{E}_{n}=(\bigcup_{l=1}^{M_n} E_{n,l})^\mathrm{c}$,  $\overline{v}_{n}= w_n \mathbbm{1}_{\overline{E}_{n}} $, and let  $\overline{\mathbf{w}}=(\overline{v}_{1},\overline{v}_{2},\ldots,\overline{v}_{N})$. It can be seen that $\mathbf{w}=\overline{\mathbf{w}} +\sum_{n=1}^{N}\sum_{l=1}^{M_n} \mathbf{v}_{n,l} $.  Define $\boldsymbol{y}_{n,l}=F(\mathbf{v}_{n,l},\mathbf{p})$. Since the $\boldsymbol{y}_{n,l}$  are at least $M+1$ vectors in $\mathbb{R}^M$,  they are linearly dependent. Consequently, there exist constants $\alpha_{n,l}\in\mathbb{R}$, with at least one of them being nonzero, such that 
\begin{equation}
\sum_{n=1}^{N}\sum_{l=1}^{M_n} \alpha_{n,l}\boldsymbol{y}_{n,l} =\boldsymbol{0}.
\end{equation}
For $n=1,2,\ldots,N$, define $\mu_n=\sum_{l=1}^{M_n} \alpha_{n,l} v_{n,l} $ and $\boldsymbol{\mu}=(\mu_1,\mu_2,\ldots,\mu_{N})$. Also, denote  $\epsilon_{\max}=\frac{1}{\max_{n,l} |\alpha_{n,l}|}>0$. For any ${\epsilon \in (-\epsilon_{\max},\epsilon_{\max})}$, we have that $1+\epsilon \alpha_{n,l} >0$ for all $n=1,2,\ldots,N$ and   $l=1,2,\ldots,M_n$. We   also see that
\begin{equation}
F(\boldsymbol{\mu},\mathbf{p})=\sum_{n=1}^{N}\sum_{l=1}^{M_n} \alpha_{n,l} \boldsymbol{y}_{n,l} = \boldsymbol{0}.
\end{equation}
Now,  for any $\epsilon\in (-\epsilon_{\max},\epsilon_{\max})$, we have that $F(\mathbf{w}+\epsilon \boldsymbol{\mu},\mathbf{p}) =   F(\mathbf{w},\mathbf{p}) \in \mathcal{C} $ and, therefore, $(\mathbf{w}+\epsilon \boldsymbol{\mu},\mathbf{p})\in \mathcal{U}$. Moreover,
\begin{equation}
\mathbf{w}+\epsilon \boldsymbol{\mu} = \overline{\mathbf{w}}+   \sum_{n=1}^{N} \sum_{l=1}^{M_n} (1+\epsilon \alpha_{n,l}) \mathbf{v}_{n,l}.
\end{equation}
Note that the $n$th element  of $\mathbf{w}_\mathrm{c}$ has support $E_{n,\mathrm{c}}$. Moreover,   the  $n$th element of $v_{n',l}$ has support $E_{n,l}$ for $n'=n$ and has empty support otherwise. Therefore, the $n$th entries have disjoint supports, which allows  us to write that
\begin{align}
\|\mathbf{w}+\epsilon \boldsymbol{\mu} \|_\mathcal{M}&=\sum_{n=1}^{N} \|\overline{v}_{n} + \sum_{l=1}^{M_n} (1+\epsilon \alpha_{n,l})\mathbf{v}_{n,l}  \|_\mathcal{M} \nonumber \\&= \sum_{n=1}^{N} \|\overline{v}_{n}\|_\mathcal{M} + \sum_{n=1}^{N}\sum_{l=1}^{M_n} (1+\epsilon \alpha_{n,l}) \|v_{n,l}\|_\mathcal{M} \nonumber \\& = \beta + \epsilon \sum_{n=1}^{N}\sum_{l=1}^{M_n} \alpha_{n,l} \|v_{n,l}\|_\mathcal{M}.
\end{align}
For sufficiently small values of $\epsilon$, this gives either $\|\mathbf{w}+\epsilon \boldsymbol{\mu} \|_\mathcal{M}<\beta$ or $\|\mathbf{w}-\epsilon \boldsymbol{\mu} \|_\mathcal{M}<\beta$.  Therefore, $\sum_{n=1}^{N}\sum_{l=1}^{M_n} \alpha_{n,l} \|v_{n,l}\|_\mathcal{M}=0$, which yields that $\|\mathbf{w}+\epsilon \boldsymbol{\mu}\|_\mathcal{M}=\|\mathbf{w}-\epsilon \boldsymbol{\mu}\|_\mathcal{M}=\beta$. This shows that $(\mathbf{w}+\epsilon \boldsymbol{\mu},\mathbf{p}),(\mathbf{w}-\epsilon \boldsymbol{\mu},\mathbf{p}) \in \mathcal{V}$, which contradicts   that $(\mathbf{w},\mathbf{p})$ is an extreme point. Therefore $\mathbf{w}$, is nonzero at most in M points, which yields the form of \eqref{epform}. Computing the norm of such an extreme point results in  
\begin{equation}
\|\mathbf{w}\|_\mathcal{M} = \sum_{n=1}^{N}\sum_{l=1}^{M_n} |a_{n,l}| \|\delta(\cdot-x_{n,l})\|_\mathcal{M} = \sum_{n=1}^{N}\sum_{l=1}^{M_n} |a_{n,l}|,
\end{equation}
which completes the proof. 
\end{proof}
\section{Implementation Details of The Numerical Example}\label{App:Detail}
The ground-truth signal for our experiment is a piecewise linear function with four segments that connects five points, located at $\{(0,0), (0.45,0), (\frac{7}{15},-2), (\frac{8}{15},2), (1,2)\}$.
We then sample data from the model $y_m = f(x_m) + \epsilon_m,m=1,\ldots,M$,  where $\epsilon_m \sim \mathcal{N}(0,\sigma^2)$ is i.i.d.\ Gaussian with  $\sigma=0.1$. We formed two training datasets of size $M=100$. In the first one, $x_m$ are i.i.d.\ samples of a uniform distribution over $[0,1]$. In the second case, we put a gap in the training dataset by sampling $x_m$ uniformly over $[0,1]\backslash [0.6,0.8]$. 

We use Gaussian kernels in the RKHS-based methods and  super-exponential kernels with $\alpha=1.99$ in the gTV-based methods. We have set $\alpha=1.99$ to have similar (near-Gaussian) kernel shapes in all cases. All methods have access to ten different width parameters from $10$ to $10^5$ in log scale. 

We set the data fidelity  to be the quadratic term $\mathrm{E}(x,y)=(x-y)^2$  in all cases except for MKL, since the SimpleMKL toolbox \cite{rakotomamonjy2008simplemkl} uses the $\epsilon$-insensitive SVM loss. The other methods are implemented using the GlobalBioIm library \cite{Soubies2019GlobalBioIm} and the codes are all available online\footnote{https://github.com/Biomedical-Imaging-Group/Multi-Kernel-Regression-gTV-}. 
 In the gTV-based methods, we have used the  multiresolution strategy of  \cite{debarre2019b} to control the accuracy.   More precisely, we start by considering 16 equi-spaced kernels and we then use FISTA to solve the convex problem of finding the corresponding kernel coefficients. The solution is propagated as initialization of a finer grid (with 32 kernels) and we continue until we reach to the finest scale, with $1,\!024$ kernels.  

 Finally, to have a fair comparison, we optimize the hyper-parameters of each method by following a standard K-fold cross-validation scheme, setting $K=5$ in our example. This includes a tuning of the regularization parameter $\lambda$ for all methods. In addition, we tune  the  width of the kernel function in single-kernel schemes   so that all methods have access to the same family of kernel functions. For computing the test error, we consider a very fine grid with stepsize $10^{-4}$ over $[0,1]$ and we compute the MSE between the learned function and the ground-truth signal.

\bibliographystyle{siamplain}
\bibliography{ref.bib}

\end{document}